\documentclass[11pt]{article}
\usepackage[a4paper,bindingoffset=0.2in,%
            left=0.9in,right=0.9in,top=1in,bottom=1in,%
            footskip=.25in]{geometry}

\usepackage[utf8]{inputenc} 
\usepackage[T1]{fontenc}    
\usepackage[draft]{hyperref}       
\usepackage{url}            
\usepackage{booktabs}       
\usepackage{amsfonts}       
\usepackage{nicefrac}       
\usepackage{microtype}      

\usepackage{amsmath,amsthm}
\usepackage{graphicx}
\usepackage{latexsym}                
\usepackage{dsfont}                  
\usepackage{mathtools}               
\usepackage{xcolor}                  
\usepackage{subfigure}
\usepackage{pgf}
\usepackage{pgfpages}                

\usepackage[printonlyused]{acronym}
\usepackage{amssymb,amsmath}
\usepackage{bm}
\newcommand{\mbf}{\mathbf}
\newcommand{\mbb}{\mathbb}
\newcommand{\mcl}{\mathcal}
\newcommand{\mrm}{\mathrm}
\newcommand{\real}{\mathbb R}
\newcommand{\T}{\text{T}}
\newcommand{\ma}[1]{{\mathbf #1}}
\renewcommand{\epsilon}{\varepsilon}
\DeclareMathOperator*{\argmin}{argmin}
\DeclareMathOperator*{\argmax}{argmax}

\def\N{\mathbb N}

\newtheorem{proposition}{Proposition}
\newtheorem{theorem}{Theorem}
\newtheorem{problem}{Problem}
\newtheorem{definition}{Definition}

\usepackage{algorithm,algorithmic}
\usepackage[colorinlistoftodos,prependcaption]{todonotes} 

\usepackage{epstopdf}
\definecolor{myred}{rgb}{0,0,0}
\definecolor{myred2}{rgb}{0,0,0}
\definecolor{myblue}{rgb}{0,0,0}
\definecolor{mygreen}{rgb}{0,0,0}

\title{Perturbation Analysis of Learning Algorithms: \\A Unifying Perspective on Generation \\ of Adversarial Examples}

\author{Emilio Rafael~Balda,
        Arash~Behboodi,
        and~Rudolf~Mathar
\thanks{Institute for Theoretical Information Technology (TI), RWTH Aachen University.}}

\begin{document}
\acrodef{DNN}{Deep Neural Network}
\acrodef{FGSM}{Fast Gradient Sign Method}
\acrodef{GNM}{gradient-based norm-constrained method}
\acrodef{BALDA}{Bounded Attack with Linearly Directed Approximation}
\acrodef{BIM}{Basic Iterative Method}
\acrodef{PGD}{Projected Gradient Descent}

\maketitle

\begin{abstract}
Despite the tremendous success of deep neural networks in various learning problems, it has been observed that adding an intentionally designed adversarial perturbation to inputs of these architectures leads to erroneous classification with high confidence in the prediction. In this work, we propose a general framework based on the perturbation analysis of learning algorithms which consists of convex programming and is able to recover many current adversarial attacks as special cases. The framework can be used to propose novel attacks against learning algorithms for classification and regression tasks under various new constraints with closed form solutions in many instances. In particular we derive new attacks against classification algorithms which are shown to achieve comparable performances to notable existing attacks. The framework is then used to generate adversarial perturbations for regression tasks which include single pixel and single subset attacks. By applying this method to autoencoding and image colorization tasks, it is shown that adversarial perturbations can effectively perturb the output of regression tasks as well.
\end{abstract}

%
\section{Introduction}
\label{sec:intro} 

\acp{DNN} excelled in recent years in  many learning tasks and demonstrated outstanding achievements in  speech analysis \cite{hinton_deep_2012} and visual tasks \cite{krizhevsky2012imagenet,he_deep_2016,szegedy_going_2015,ren_faster_2017}.
Despite their success, they have been shown to suffer from instability in their classification under adversarial perturbations \cite{szegedy_intriguing_2014}. Adversarial perturbations are intentionally worst case designed noises that aim at changing the output of a \ac{DNN} to an incorrect one. 
The explosion of research during past years makes it almost impossible to refer to all important works in this area and do justice to all excellent works. 
However, we refer to several important results from the literature, that are highly connected to this paper.

Although \acp{DNN} might achieve robustness to random noise \cite{fawzi_robustness_2016}, it was shown that there is a clear distinction between the robustness of a classifier to random noise and its robustness to adversarial perturbations.  The existence of adversarial perturbations was known for machine learning algorithms \cite{barreno_security_2010}, however, they were first noticed in deep learning research in \cite{szegedy_intriguing_2014}. 
The peculiarity of adversarial perturbations lied in the fact that they managed to fool state of the art networks \textcolor{myblue}{into} making confident and wrong decisions in classification tasks, and they, nevertheless, appeared unperceived to \textcolor{myblue}{the} naked eye. These discoveries gave rise to extensive research on understanding the instability of \acp{DNN}, exploring various attacks and devising multiple defenses (for instance refer to \cite{akhtar_threat_2018,wang_theoretical_2017,fawzi_fundamental_2015} and references therein).
%
%
\textcolor{myblue}{Most adversarial} attacks fall generally  into two \textcolor{myblue}{classes,} white-box and black-box attacks. In white-box attacks, the attacker knows completely the architecture of the target algorithms and additionally, there are attacks with partial knowledge of the architecture.  
However, black-box attacks require no information about the target neural network, see for instance \cite{sarkar2017upset}. 
In this work, the focus is on white-box attacks. 
The overall aim of attacks, as in \cite{goodfellow_explaining_2014,moosavi2016deepfool,rao_universal_2012},
is to apply perturbations to the system inputs, that are not perceived by the system's administrator, such that the performance of the system is severely degraded. 

{Adversarial} perturbations were obtained  in \cite{szegedy_intriguing_2014} to maximize the prediction error at the output and were approximated using box-constrained L-BFGS. 
{The \ac{FGSM}} in \cite{goodfellow_explaining_2014} was based on finding the scaled sign of the gradient of the cost function. 
Note that the \ac{FGSM} aims at minimizing $\ell_\infty$-norm of the perturbation while the former algorithm minimizes $\ell_2$-norm of the perturbation under box constraint on the perturbed example. 

More effective attacks utilize either iterative procedures or randomizations. The algorithm DeepFool \cite{moosavi2016deepfool} conducts an iterative linearization of the \ac{DNN} to generate perturbations that are minimal in the $\ell_p$-norm for $p>1$. 
In \cite{BIM2016} the authors propose an iterative version of \ac{FGSM}, called \ac{BIM}.
This method was later extended in \cite{pgd_attack}, where randomness was introduced in the computation of adversarial perturbations. 
This attack is called the \ac{PGD} method and {was} employed in \cite{pgd_attack} to devise a \textcolor{myred}{defense} against adversarial examples.  
An iterative algorithm based on \ac{PGD} combined with randomization was introduced in  \cite{athalye2018obfuscated} and has been used to dismantle many defenses so far \cite{athalye_robustness_2018}. Another popular way of generating adversarial examples is by constraining the $\ell_0$-norm of the perturbation. These types of attacks are known as single pixel attacks \cite{su2017one} and multiple pixel attacks \cite{papernot2016limitations}. 
%


An interesting feature of these perturbations is their generalization over other datasets and \acp{DNN} \cite{rao_universal_2012,goodfellow_explaining_2014}. 
These perturbations are called universal adversarial perturbations. 
This is partly explained by the fact that certain underlying properties of the perturbation, such as direction in case of image perturbation, matters the most and \textcolor{myblue}{is} therefore generalized through different datasets.
For example, the attack from \cite{ensemble_attack} shows that adversarial examples transfer from one random instance \textcolor{myred}{of a neural} network to another. 
In that work, the authors showed the effectiveness of {these} types of attacks for enhancing the robustness of neural networks, since they provide diverse perturbations during adversarial training. Moreover, \cite{moosavi2017universal} showed the existence of universal adversarial perturbations that are independent from the system and the target input.

Since the rise of adversarial examples for image classification, novel algorithms have been developed for attacking other types of systems. In the field of computer vision, \cite{metzen2017universal} constructed an attack on image segmentation, while \cite{segmentation} designed attacks for object detection. The Houdini attack \cite{cisse2017houdini} aims at distorting speech recognition systems. Moreover, \cite{papernot2016crafting} taylored an attack for recurrent neural networks, and \cite{lin2017tactics} for reinforcement learning. Adversarial examples exist for probabilistic methods as well. 
For instance, \cite{kos2017adversarial} showed the existence of adversarial examples for generative models. For regression problems, \cite{tabacof2016adversarial} designed an attack that specifically targets variational autoencoders. 

%


There are various theories regarding the nature of adversarial examples and the subject is heavily investigated. Initially, the authors in \cite{goodfellow_explaining_2014} proposed the linearity hypothesis where the existence of adversarial images is attributed to the approximate linearity  of classifiers, although this hypothesis has been challenged in \cite{tanay_boundary_2016}. 
Some other theories focus mostly on decision boundaries of classifiers and their analytic properties \cite{fawzi_robustness_2016,fawzi_robustness_2017}. The work from \cite{raghunathan2018certified} provides a framework for determining the robustness of a classifier against adversarial examples with some performance guarantees. For a more recent theoretical approach to this problem refer to \cite{tsipras_there_2018}.


There exist several types of defenses against adversarial examples, as well as subsequent methods for bypassing them. 
For instance, the authors in \cite{carlini2017towards} proposed three attacks to bypass defensive distillation of the adversarial perturbations \cite{papernot2016distillation}.
Moreover, the attacks from \cite{athalye2018obfuscated}, bypassed 7 out of 9 non-certified defenses of ICLR 2018 that claimed to be white-box secure.  
The most common defense is adding adversarial examples to the training set, also known as adversarial training. For that purpose different adversarial attacks may be employed.  Recently, the PGD attack is used in \cite{pgd_attack} to provide the state of the art defense against adversarial examples for various image classification datasets. 
%
\subsection{Our Contribution}

In this work, we focus on the generation of adversarial examples with a sufficiently general framework that includes many existing attacks and can be easily extended to generate new attacks for different scenarios. \textcolor{mygreen}{We  build upon our previous work \cite{BaBeMa18b} to introduce} a connection between perturbation analysis of learning algorithms and adversarial perturbations. This leads to a general formulation for the problem of generating adversarial examples using convex programming. The general framework includes many existing attacks as special cases, provides closed form solutions and can be easily extended to generate new algorithms. In particular, we derive novel algorithms for designing adversarial attacks for classification which are benchmarked with state of the art attacks.

Another contribution of this paper is to employ this framework in context of adversarial perturbations for regression problems, a topic that has  not been yet widely explored. Regression loss functions differ from classification loss functions in that it is sufficient to maximize the output perturbation, for instance measured in $\ell_2$-norm. In classification tasks, such a maximization might not necessarily change the output label particularly because these perturbations might push the instances far away from classification margins. There is no natural margin in regression tasks. We address various technical difficulties of this problem and use our framework to generate adversarial examples for regression tasks. In particular single pixel and single subset attacks are discussed. It is shown that this problem is related to the MaxCut problem and hence difficult to solve. We propose a greedy algorithm to \textcolor{myred}{overcome this issue}. 

Finally, the proposed algorithms are experimentally evaluated using state of the art benchmarks for classification and regression tasks. It is shown that our proposed method achieves comparable and sometimes better performance than many existing attacks for classification problems. Furthermore, it is shown that regression tasks such as image colorization and autoencoding suffer from adversarial perturbations as well.

%
%
\section{Fooling Classifiers with First-Order Perturbation Analysis } \label{sec:avd_and_rub}
The perturbation analysis, also called sensitivity analysis, is used in signal processing for analytically quantifying the error at the output of a system that occurs as consequence of a known perturbation at the system's input. Adversarial images can also be considered as a slightly perturbed version of original images that manage to change the output of the classifier. 
Indeed, the generation of adversarial examples  in \cite{moosavi2016deepfool,goodfellow_explaining_2014} is implicitly based on maximizing  the effect of an input perturbation on a relevant function which is either the classifier function or the cost function used for training. In \textcolor{myblue}{the \ac{FGSM},} given in \cite{goodfellow_explaining_2014}, the perturbation at the output of the training cost function is first analyzed using first-order perturbation analysis of the cost function and then maximized to fool the algorithm. The \mbox{DeepFool} method, given in \cite{moosavi2016deepfool}, maximizes the output perturbation for the linearized approximation of the underlying classifier which is indeed its first order-perturbation analysis. We develop further the connection between  perturbation analysis and adversarial examples in this section. 

\subsection{Adversarial Perturbation Design}
As it was mentioned above, 	 
adversarial examples can be considered as perturbed version of training examples by an adversarial perturbation $\bm \eta$. The perturbation analysis of classifiers is particularly difficult in general since the classifier function maps inputs to discrete set of labels and therefore it is not differentiable. Instead, the classification problem is slightly modified as follows.

\begin{definition}[Classification]
A classifier is defined by the mapping  $k: \real^{M} \rightarrow [K]$\footnote{We denote the set $\{1,\dots,n\}$ by $[n]$ for $n\in\N$.} that maps an input $\mathbf{x} \in \real^{M}$ to its estimated class $k\left(\mathbf{x}\right) \in [K]$. The mapping $k(\cdot)$ is itself defined by 
\begin{equation} \label{eq:classifier}
k(\mathbf{x}) = \argmax_{l \in[K] } \left\lbrace f_l \left(\mathbf{x} \right)  \right\rbrace \,,
\end{equation}
where $f_l(\mathbf{x}):\real^M\to \real$'s are called score functions representing the probability of class belonging. 
\end{definition}

The function $f(\mathbf{x})$ given by the vector $(f_1(\mathbf{x}),\dots,f_m(\mathbf{x}))$ can be assumed to be differentiable almost everywhere for many classifiers. 

The problem of adversarial generation consists of finding a perturbation that changes the classifier's output.
However, it is desirable for adversarial perturbations to modify training instances only in an insignificant and unnoticeable way. This is controlled by adding a constraint on the adversarial perturbation. For instance, the perturbation generated by \textcolor{myblue}{the} \ac{FGSM} is bounded in {the} $\ell_\infty$-norm and {the} \mbox{DeepFool} method directly {minimizes} the norm of the perturbation that changes the classifier's output. While \mbox{DeepFool} might generate perturbations that are perceptible, \textcolor{myblue}{the} \ac{FGSM} might not change the classifier's output.

An intriguing property of adversarial examples is that the perturbation does not \textcolor{myred}{distort} the image significantly so that the naked eye can not detect any notable change in the images. 
One way of imposing this property in adversarial design is to constrain the input perturbation to keep the output of the ground truth classifier, also called oracle classifier \cite{wang_theoretical_2017}, intact. The oracle classifier represents the naked eye  in case of image classification. 
The score functions of the oracle classifier are denoted by $g_l(\cdot)$. 
The undetectability constraint for an adversarial perturbation $\bm{\eta}$ is formulated as
\begin{equation}\label{eq:loss_ours}
    L_g(\mathbf{x}, \bm{\eta}) = g_{k(\mathbf{x})}(\mathbf{x} + \bm{\eta}) - \max_{l \neq k(\mathbf{x})} g_l(\mathbf{x} + \bm{\eta}) >0\, .
\end{equation}

Therefore the problem of adversarial design can be formulated as follows.

\begin{problem}[Adversarial Generation Problem]
	For a given $\mbf x\in\real^M$, find a perturbation $\bm{\eta} \in \real^{M}$ to fool the classifier $k(\cdot)$ by the adversarial sample $\hat{\mathbf{x}} = \mathbf{x} + \bm{\eta}$ such that $k(\mathbf{x}) \neq k(\hat{\mathbf{x}})$ and the oracle classifier is not changed, i.e.,
\begin{equation} 
\tag{AGP}
\begin{split}
 \mathrm{Find:}& \quad \bm\eta\nonumber\\
 \mathrm{s.t.} &\quad  L_f(\mathbf{x}, \bm{\eta}) = f_{k(\mathbf{x})}(\mathbf{x} + \bm{\eta}) - \max_{l \neq k(\mathbf{x})} f_l(\mathbf{x} + \bm{\eta}) <0\\
  &\quad  L_g(\mathbf{x}, \bm{\eta}) = g_{k(\mathbf{x})}(\mathbf{x} + \bm{\eta}) - \max_{l \neq k(\mathbf{x})} g_l(\mathbf{x} + \bm{\eta}) >0
\end{split}
  \label{eq:originalOP}
 \end{equation}

\end{problem}

The problem \eqref{eq:originalOP} is too general to be useful in practice directly. 
Next we explore different methods for making {this} problem tractable \textcolor{myblue}{ in some} cases of interest. 
Since $\mbf x$ and $f$ are fixed for the attacker, we simplify the notation by dropping the subscript $f$ and assuming that gradients are always with respect to $\bm \eta$, that is $L(\mbf x, \cdot) = L_f(\mbf x, \cdot)$ and $\nabla L(\mbf x, \cdot) = \nabla_{\bm \eta} L_f(\mbf x, \cdot)$. We keep these shorthand notations throughout the paper.

\subsection{Perturbation Analysis}

There are two problems with the above formulation. First, the oracle function is not known in general and second the function $L(\mbf x, \cdot)$ can be non-convex. 
One solution is to approximate $L(\mbf x, \cdot)$ with  a tractable function like linear functions which can be obtained through perturbation analysis of each individual function. 
The constraint on the oracle function can also be replaced with constraints on the perturbation itself, {for instance by} imposing upper bounds on the $\ell_p$-norm of the perturbation. 
Different classes of attacks can be obtained for different choices of $p$ and are well known in the literature such as $\ell_\infty$-attacks, $\ell_2$-attacks and $\ell_1$-attacks (see the survey in \cite{akhtar_threat_2018} for details).

The first order perturbation analysis of $L$ yields 
\[
L(\mbf x, \bm{\eta}) = L(\mbf x, \mbf 0) + \bm{\eta}^{\mathrm{T}} \nabla L(\mbf x, \mbf 0) + \mcl O(\|\bm \eta\|_2^2),
 \]
where 
$\mcl O(\|\bm \eta\|_2^2)$ contains higher order terms. 
The condition that corresponds to the oracle function can be approximated by
$\| \bm{\eta} \|_{p} \leq \epsilon$ for  sufficiently small $\epsilon \in \real^{+}$. 
This means that the noise is sufficiently small in $\ell_p$-norm sense so that the observer does not notice it. 
These gradient and norm relaxations yield the following alternative optimization problem
\begin{align}
 \text{Find:}& \quad \bm\eta\nonumber\\
 \text{s.t.} &\quad  L(\mbf x, \mathbf{0}) + \bm{\eta}^{\mathrm{T}} \nabla L(\mbf x, \mathbf{0})<0, \quad \| \bm{\eta} \|_{p} \leq \epsilon.
 \label{eq:originalOPII}
\end{align}
The above problem was also derived in \cite{robustnessguarantees2017} and is a convex optimization problem that can be efficiently solved. As we will see later, this formulation of the problem can be relaxed into some well known existing adversarial methods. 
However it is interesting to observe that this problem is not always feasible as stated in the following proposition. 
\begin{theorem}
 The optimization problem \eqref{eq:originalOPII} is not feasible if for $q=\frac{p}{p-1}$
 \begin{equation}
  \epsilon \|\nabla L(\mbf x, \mathbf{0})\|_q < L(\mbf x, \mathbf 0).
  \label{eq:condition}
 \end{equation}
 \label{prop:feasibility}
\end{theorem}
\begin{proof}
The proof follows a simple duality argument and is an elementary optimization theory result. \textcolor{myred}{A similar result can be inferred from } \cite{robustnessguarantees2017}. We repeat the proof for completeness. Note that the dual norm of $\ell_p$ is defined by
\[
 \|\ma x\|_p^*=\sup\{\ma a^\T \ma x: \|\ma a\|_p\leq 1\}.
\]
 Furthermore $ \|\ma x\|_p^*= \|\ma x\|_q$ for $q=\frac{p}{p-1}$.
Since the $\ell_p$-norm of $\bm\eta$ is bounded by $\epsilon$, the value of $\bm{\eta}^{\mathrm{T}} \nabla L(\mbf x, \mathbf{0})$ is always bigger than $-\epsilon \|\nabla L(\mbf x, \mathbf{0})\|_p^*$. However if the condition \eqref{eq:condition} holds, then we have
\[
 L(\mbf x, \mathbf{0}) + \bm{\eta}^{\mathrm{T}} \nabla L(\mbf x, \mathbf{0})\geq L(\mbf x, \mathbf{0}) -\epsilon \|\nabla L(\mbf x, \mathbf{0})\|_p^*>0.
\]
Therefore, the problem is not feasible.
\end{proof}
 
Theorem \ref{prop:feasibility} shows that given a vector $\mathbf x$, the adversarial perturbation should have at least $\ell_p$-norm equal to $\frac{ L(\mbf x, \mathbf 0)  }{\|\nabla L(\mbf x, \mathbf{0})\|_q}$. In other words if the ratio $\frac{ L(\mbf x, \mathbf 0)  }{\|\nabla L(\mbf x, \mathbf{0})\|_q}$ is {small}, then it is easier to fool the network by the $\ell_p$-attacks. In that sense,  Theorem \ref{prop:feasibility} provides an insight into the stability of classifiers. In \cite{moosavi2016deepfool}, the authors suggest that the robustness of the classifiers can be measured as
\[
 \hat{\rho}_{1}(f)=\frac{1}{|\mathcal D|}\sum_{\mathbf x\in\mathcal D}\frac{\|\hat{\mathbf r}(\mathbf x)\|_p}{\|\mathbf x\|_p},
\]
where $\mathcal D$ denotes the test set and $\hat{\mathbf r}(\mathbf x)$ is the minimum perturbation required to change the classifier's output. The above theorem suggests that one can also use the following as the measure of robustness
\[
 \hat{\rho}_{2}(f)=\frac{1}{|\mathcal D|}\sum_{\mathbf x\in\mathcal D}\frac{ L(\mathbf x, \mbf 0)  }{\|\nabla L(\mathbf{x}, \mbf 0)\|_q}.
 \]
 The lower $ \hat{\rho}_{2}(f)$, the easier it gets to fool the classifier and therefore it becomes less robust to adversarial examples.  One can also look at other statistics related to $\frac{ L(\mathbf x, \mbf 0)  }{\|\nabla L(\mathbf{x}, \mbf 0)\|_q}$ in order to evaluate the robustness of classifiers.
%
%

Theorem \ref{prop:feasibility} shows that the optimization problem \eqref{eq:originalOPII} might not be feasible. We propose to get around this issue by solving an optimization problem which keeps only one of the constraints, depending on the scenario, and
selects an appropriate the objective function to preserve the other constraint as much as possible. The objective function in this sense models the deviation from the constraint and is minimized in the optimization problem. We consider two optimization problems for this purpose. 

First, the norm-constraint on the perturbation is preserved. The following optimization problem, called \ac{GNM}, aims at minimizing $ L(\mathbf{x}, \mbf 0) + \bm{\eta}^{\mathrm{T}} \nabla L(\mathbf{x}, \mbf 0)$ \textcolor{myred}{by solving the following problem}:
\begin{equation}\label{eq:MainOpt}
\min_{\bm{\eta}} \left \{ L(\mathbf{x}, \mbf 0) + \bm{\eta}^{\mathrm{T}} \nabla L(\mathbf{x}, \mbf 0) \right \} \quad \mathrm{s.t.} \quad \| \bm{\eta} \|_{p} \leq \epsilon \, .
\end{equation}
This method finds the best perturbation under the norm-constraint. The constraint aims at guaranteeing that the adversarial images \textcolor{myred}{are} still imperceptible by an ordinary observer. Note that \eqref{eq:MainOpt} is fundamentally different from \cite{moosavi2016deepfool, robustnessguarantees2017}, where the norm of the noise does not appear as a constraint. Using a similar duality argument, the problem \eqref{eq:MainOpt} has a closed form solution given below.
\begin{theorem}
  If $\nabla L(\mathbf{x}, \bm{\eta})=(\frac{\partial L(\mathbf x, \bm{\eta})}{\partial \eta_1},\dots,\frac{\partial L(\mathbf x, \bm{\eta})}{\partial \eta_M} )$, the closed form solution to the minimizer of the problem \eqref{eq:MainOpt} is given by
\begin{align}
   & \bm{\eta} = - \epsilon \frac{ 1}{\|\nabla L(\mathbf{x}, \mbf 0)\|_q^{q-1}} 
 \mathrm{sign}(\nabla L(\mathbf{x}, \mbf 0)) \odot |\nabla L(\mathbf{x}, \mbf 0)|^{q-1}
 \label{eq:minMainOP}
\end{align}
for $q=\frac{p}{p-1}$, {where $\mathrm{sign}(\cdot)$ and $|\cdot|^{q-1}$ are applied element-wise, and $\odot$ denotes the element-wise (Hadamard) product}. 
Particularly for $p=\infty$, we have $q=1$ and the solution is given by the following
\begin{equation}
    \bm{\eta} = - \epsilon \, \mathrm{sign}(\nabla L(\mathbf{x}, \mbf 0) ) \, .
\end{equation}
\label{thm:AdvNormSol}
\end{theorem}
\begin{proof}
Based on the duality argument from convex analysis, it is known that 
\[
 \sup_{\|\bm \eta\|_p\leq 1}\bm{\eta}^{\mathrm{T}} \nabla L(\mathbf{x}, \mbf 0)=\|\nabla L(\mathbf{x}, \mbf 0)\|_p^*,
\]
where $\|\cdot\|^*$ is the dual norm.  This implies that the objective function is lower bounded by $L(\mathbf{x}, \mbf 0) -\epsilon \|\nabla L(\mathbf{x}, \mbf 0)\|_p^*$. It is easy to verify that the minimum is attained by {\eqref{eq:minMainOP}}. 
\end{proof}

The advantage of \eqref{eq:MainOpt}, apart from being convex and enjoying computationally efficient solutions, is that one can incorporate other convex constraints into the optimization problem to guarantee additional required properties of the perturbation. Note that the introduced method in \eqref{eq:MainOpt} can also be used for other target functions or learning problems. If the training cost function is maximized under \textcolor{myred}{a} norm constraint, as in \cite{goodfellow_explaining_2014}, the solution of  \eqref{eq:MainOpt} with $p=\infty$ recovers the adversarial perturbations obtained via \textcolor{myblue}{the \mbox{FGSM}}. The problem \eqref{eq:MainOpt} guarantees that the perturbation is small, however, it might not change the classifier's output.


The second optimization problem, on the other hand, preserves the constraint for changing the classifier's output and minimizes the perturbation norm instead. The feasibility problem of \eqref{eq:originalOPII} can therefore be simplified to 
\begin{equation}\label{eq:MainOptII}
\min_{\bm{\eta}} \| \bm{\eta} \|_{p} \quad \mathrm{s.t.} \quad L(\mathbf{x}, \mbf 0) + \bm{\eta}^{\mathrm{T}} \nabla L(\mathbf{x}, \mbf 0)    \leq 0 \, ,
\end{equation}
which recovers the result in \cite{moosavi2016deepfool} although without the iterative procedure. This problem has a similar closed form solution.
\begin{proposition}
If $\nabla L(\mathbf{x}, \bm{\eta})=(\frac{\partial L(\mathbf x, \bm{\eta})}{\partial \eta_1},\dots,\frac{\partial L(\mathbf x, \bm{\eta})}{\partial \eta_M} )$, the closed form solution to the problem \eqref{eq:MainOptII} is given by
\begin{align}
   & \bm{\eta} = -  \frac{ L(\mbf x, \mbf 0)}{\|\nabla L(\mathbf{x}, \mbf 0)\|_q^{q-1}}
   \mathrm{sign}(\nabla L(\mathbf{x}, \mbf 0)) \odot |\nabla L(\mathbf{x}, \mbf 0)|^{q-1}
 \label{eq:minMainOPII}
\end{align}
for $q=\frac{p}{p-1}$. 
\label{prop:AdvNormSol}
\end{proposition}

Note that the perturbation found in Proposition \ref{prop:AdvNormSol},  like the solution to \ac{GNM}, aligns with the gradient of the classifier function and they only differ in their norm. Although the perturbation in \eqref{eq:minMainOPII}, unlike  the solution to \ac{GNM}, is able to fool the classifier, the perturbation in \eqref{eq:minMainOPII} might be perceptible by the oracle classifier.
There are other variants of adversarial generation methods that rely on an implicit perturbation analysis of a relevant function. These methods can be easily obtained by small modification of the methods above. 

Iterative procedures can be easily adapted to the current formulation by repeating the optimization problem until the classifier output changes while keeping the perturbation small at each step. Later we provide an iterative version of {the} \ac{GNM} and compare it with  {DeepFool \cite{moosavi2016deepfool}, as well as other methods}.

Another class of methods relies on introducing randomness in the generation process. A notable example is \textcolor{myblue}{the \ac{PGD} attack} introduced in \cite{pgd_attack} which is one of the state of the art attacks.  The first-order approximation is then taken around another point $\tilde{\bm \eta}$ with $\tilde \epsilon \triangleq \| \tilde{\bm \eta} \|_p \leq \epsilon$. 
In other words we approximate $L(\mbf x, \cdot)$ by a linear function around the point $\tilde{\bm \eta}$ within an $\tilde \epsilon$-radius from $\bm \eta = \mbf 0$. 
This new point $\tilde {\bm \eta}$ can be computed at random using arbitrary distributions with $\ell_p$-norm bounded by $\epsilon$.
Changing the center of the first order approximation from $\mbf 0$ to $\tilde{\bm \eta}$ does not change the nature of the problem since
$
L(\mbf x, \bm \eta) \approx L({\mbf x}, \tilde{\bm \eta}) + ( \bm \eta - \tilde{\bm \eta})^\T  \nabla L({\mbf x}, \tilde{\bm \eta}) 
$
leads to the following problem
\begin{align*}
&\quad \min_{\bm \eta} L({\mbf x}, \tilde{\bm \eta}) + ( \bm \eta - \tilde{\bm \eta})^\T  \nabla L({\mbf x}, \tilde{\bm \eta})  \quad \mathrm{s.t.} \quad \|\bm \eta \|_p\leq \epsilon,
\end{align*}
which is equivalent to:
\begin{align}
\min_{\bm \eta} \bm\eta^\T  \nabla L({\mbf x}, \tilde{\bm \eta}) \quad \mathrm{s.t.} \quad \|\bm \eta \|_p\leq \epsilon \, . \label{eq:MainOpt_lin_tilde}
\end{align}

From this result one can add randomness to the computation of adversarial examples by selecting \textcolor{myred}{$\tilde{\bm \eta}$} in a random fasion. 
This is desirable when training models with adversarial examples since it increases the diversity of the adversarial perturbations during training \cite{ensemble_attack}.
\section{From Classification to Regression}\label{sec:system}
%
In classical statistical learning theory, \textcolor{myred}{regression problems are defined} in the following manner.
Given $N \in \mbb N$ samples $\{(\mbf x_i, \mbf y_i)\}_{i=1}^{N}$ drawn according to some unknown distribution $P_{X, Y}$, a regression model computes a function $f: \mbb R^M \rightarrow \mbb R^K$ that aims to minimize the expected loss $\mbb E_{P}( \mcl L(f(\mbf x), \mbf y) )$, where $\mcl L:\mbb R^M \times \mbb R^K \rightarrow \mbb R$ is a function that measures the similarity between $f(\mbf x)$ and $\mbf y$. While logarithmic losses are popular in classification problems, the squared loss $\mcl L(f(\mbf x), \mbf y) = \| f(\mbf x) - \mbf y\|_2^2$ is mostly used for the general regression setting. 
For the sake of notation, given $\mbf y$ and $f$, let us \textcolor{myblue}{redefine $L(\mbf x, \bm \eta)$ as} $L(\mbf x, \bm \eta) = \mcl L(f(\mbf x + \bm \eta), \mbf y)$. 

For a given $f$, $\mbf x$ and $\mbf y$, an adversarial attacker finds an additive perturbation vector $\bm \eta$ that is \textit{imperceptible} to the administrator of the target system, while maximizing the loss of the perturbed input $L(\mbf x, \bm \eta)$ as 
\[
\max_{\bm \eta} L(\mbf x , \bm \eta) \quad \text{s.t.} \quad \bm \eta \text{ is imperceptible} \, .
\]
In contrast with classification problems where maximum perturbations at the output might not change the class, adversarial instances maximize the output perturbation in regression problems. 

As in \eqref{eq:MainOpt}, a constraint on the $\ell_p$-norm of $\bm \eta$ models imperceptibility leading to the following formulation of the problem
\begin{equation}\label{eq:general_opt_problem}
\max_{\bm \eta} \| \mbf y - f(\mbf x  + \bm \eta)\|^2_2  \quad \text{s.t.} \quad \| \bm \eta \|_p \leq \epsilon \, .
\end{equation}
Consider the image colorization problem where the goal is to add proper coloring on top of gray scale images. In this problem, $f(\cdot)$ is the \textcolor{myblue}{regression} algorithm and assumed to be known however the ground truth colorization $\mbf y$ is generally unknown. Without knowing $\mbf y$, the optimization problem \eqref{eq:general_opt_problem} is ill posed and cannot be solved in general. There are some cases where the output $\mbf y$ is known by the nature of the problem, for instance, when $f(\cdot)$ is an encoder-decoder pair as in autoencoders for which $\mbf y = \mbf x$. 

Since the goal is to perturb the acting \textcolor{myblue}{regression} algorithm, we can assume that $\mbf y \approx f(\mbf x)$ which means that the algorithm provides a good although not perfect approximation of the ground truth function. We use the formulation in \eqref{eq:general_opt_problem} and discuss the implications of applying the approximation $\mbf y \approx f(\mbf x)$ in later sections.

\subsection{A Quadratic Programming Problem}\label{sec:quadratic}
%
In general $f(\mbf x)$ is a non-linear and non-convex function, so we have that {$L(\mbf x , \cdot)$} is non-convex. Here again the perturbation analysis of $f(\cdot)$ can be used to relax \eqref{eq:general_opt_problem} and to obtain a convex formulation of the adversarial problem. The first order perturbation analysis of $f(\mbf x)$ yields the approximation $f(\mbf x + \bm \eta) \approx f(\mbf x) + \ma{J}_f(\mbf x)  \bm \eta$, \textcolor{myblue}{where $\ma{J}_f(\cdot) $ is the Jacobian matrix of $f(\cdot)$}. 
This approximation leads to the following convex approximation of {$L(\mbf x, \cdot)$}:
\textcolor{myred}{
\begin{align*}
L(\mbf x &, \bm \eta) \\
&\approx \| \mbf y \|_2^2 - 2 \mbf y^\T  (f(\mbf x) + \ma{J}_f(\mbf x)  \bm \eta) + \| f(\mbf x) + \ma{J}_f(\mbf x)  \bm \eta \|_2^2  \\
&= \| \mbf y \|_2^2 - 2 \mbf y^\T  f(\mbf x) + \| f(\mbf x) \|_2^2  \\
&\quad + 2 \left(f(\mbf x) - \mbf y \right)^T \ma{J}_f(\mbf x) \bm \eta  + \| \ma{J}_f(\mbf x) \bm \eta \|_2^2 \, .
\end{align*}
}
Since the first three terms of this expression do not depend on $\bm \eta$, the optimization problem from \eqref{eq:general_opt_problem} is reduced to
\textcolor{myred}{
\begin{equation}\label{eq:quadratic_other_problem}
\max_{\bm \eta} 2 \left(f(\mbf x) - \mbf y \right)^T\ma{J}_f(\mbf x) \bm \eta  + \| \ma{J}_f(\mbf x) \bm \eta \|_2^2 \quad \mathrm{s.t.} \quad \|\bm \eta \|_p\leq \epsilon \, .
\end{equation}
}
The above convex maximization problem is, in general, challenging and NP-hard. Nevertheless, since $\mbf y$ is usually not known, we may use the assumption that $\mbf y \approx f(\mbf x)$, which simplifies the problem to 
\begin{equation}
\max_{\bm \eta} \| \ma{J}_f(\mbf x) \bm \eta \|_2^2 \quad \text{s.t.} \quad \|\bm \eta \|_p \leq \epsilon \, . 
\label{eq:original_quad_opt}
\end{equation}

Although this problem is a convex quadratic maximization under an $\ell_p$-norm constraint and in general challenging, {it} can be solved efficiently in some cases. 
For general $p$, the maximum value is indeed related to the operator norm of  $\ma{J}_f(\mathbf x)$ \cite{horn_matrix_2013}. 
\textcolor{myred}{This norm is central in stability analysis of many signal processing algorithms }(for instance see \cite{foucart_mathematical_2013}). The operator norm of a matrix $\ma A\in\mathbb{C}^{m\times n}$ between $\ell_p$ and $\ell_q$ is defined as
\[
\|\ma A\|_{p\to q} \triangleq\sup_{\|\ma x\|_p\leq 1}{\|\ma A\ma x \|_q}.
\]
Using this notion, we can see that $\|\frac{\bm \eta}{\epsilon}  \|_{p}\leq 1$ leads to 
$
 \| \ma{J}_f(\mathbf x)\bm \eta  \|_{2}=\epsilon \| \ma{J}_f(\mathbf x)\frac{\bm \eta}{\epsilon}  \|_{2}\leq \epsilon \|\ma{J}_f(\mathbf x)\|_{p\to 2}.
$
Therefore, the problem of finding a solution to \eqref{eq:original_quad_opt} amounts to finding the operator norm  $\|\ma{J}_f(\mathbf x)\|_{p\to 2}$. First observe that the maximum value is achieved on the border namely for  $\|\bm \eta \|_p =\epsilon$. In the case where $p=2$, this problem has a closed-form solution.  If $\mbf v_{\max}$ is the unit $\ell_2$-norm eigenvector corresponding to the maximum eigenvalue of $\ma{J}_f(\mathbf x)^{\mathrm{T}} \ma{J}_f(\mathbf x)$, then  
\begin{equation}\label{eq:quad_sol_l2}
\bm \eta^*= \pm \epsilon  \, \mbf v_{\max}
\end{equation}
solves the optimization problem. The maximum eigenvalue of $\ma{J}_f(\mathbf x)^{\mathrm{T}} \ma{J}_f(\mathbf x)$ corresponds to the square of the  spectral norm $\|\ma{J}_f(\mathbf x)\|_{2\to 2}$.

Another interesting case is when $p=1$. In general, the $\ell_1$-norm is usually used as a regularization technique to promote sparsity. When the solution of a problem should satisfy a sparsity constraint, the direct introduction of this constraint into the \textcolor{myblue}{optimization leads} to NP-hardness of the problem. Instead the constraint is relaxed by adding $\ell_1$-norm regularization. The adversarial perturbation designed in this way tends to have only a few non-zero entries. This corresponds to  scenarios like single pixel attacks where only a few pixels are supposed to change. For this choice, we have
\textcolor{myred}{
\[
 \|\ma A\|_{1\to 2}=\max_{k\in[n]} \|\ma a_k\|_2,
\]
}
where $\ma a_k$'s are the columns of $\ma A$. Therefore, if the columns of the Jacobian matrix are given by $\ma{J}_f(\mathbf x)=[\ma J_1 \hdots \ma J_{M}] $, then
\textcolor{myred}{
\[
  \| \ma{J}_f(\mathbf x)\bm \eta  \|_{2}\leq \epsilon \max_{k\in [M]}\|\ma J_k\|_2,
\]
}
and the maximum is attained with
\textcolor{myred}{
\begin{equation}\label{eq:quad_sol_l1}
\bm\eta^*= \pm \epsilon \mbf e_{k^*}\quad\text{ for }\quad k^*=\argmax_{k\in [M]}\|\ma J_k\|_2,
\end{equation}
}
where the vector $\mbf e_i$ is the $i$-th canonical vector. For the case of gray-scale images, where each pixel is represented by a single entry of $\mbf x$, this constitutes a single pixel attack. 
Some additional constraints must be added in the case of RGB images, where each pixel is represented by a set of three values.

Finally, the case where the adversarial perturbation is bounded with the $\ell_\infty$-norm is also of particular interest. This bound guarantees that the noise entries have bounded values. 
The problem of designing adversarial noise corresponds to finding $ \|\ma{J}_f(\mathbf x)\|_{\infty\to 2}$. Unfortunately, this problem turns out to be NP-hard \cite{rohn_computing_2000}. However, it is possible to approximate this norm using semi-definite programming as proposed in \cite{hartman_tight_2015}. Semi-definite programming scales badly with input dimension in terms of computational complexity, \textcolor{myred}{namely} $O(n^6)$ with $n$ the underlying dimension, and therefore might not be suitable for fast generation of adversarial examples when the input dimension is very high. We address these problems later in Section \ref{sec:pixel}, where we obtain fast approximate solutions for $ \|\ma{J}_f(\mathbf x)\|_{\infty\to 2}$ and single pixel attacks.
%
%
\subsection{A Linear Programming Problem}\label{sec:linear}
%
The methods derived in Section \ref{sec:quadratic} suffer from one main drawback, they require storing \textcolor{myred}{$\mbf J_f(\mbf x) \in \real^{K \times M}$} into memory. While this may be doable for some applications, it is not feasible for others. For example, if the target system is an autoencoder for RGB images with size $680 \times 480$, that is $M = K = 680 \cdot 480 \cdot 3  \approx 9 \cdot 10^5$, storing $\mbf J_f(\mbf x) \in \real^{9 \cdot 10^5 \times 9 \cdot 10^5}$ requires loading around $8\cdot 10^{11}$ values into memory, which is in most cases not tractable. Note that, in order to solve \eqref{eq:original_quad_opt} for $p=2$, we would require computing the eigenvalue decomposition of $\mbf J_f(\mbf x)^\T  \mbf J_f(\mbf x)$ as well. This motivates us to relax the problem into a linear programming problem as in Section \ref{sec:avd_and_rub}, where $\mbf J_f(\mbf x)$ is computed implicitly and we do not require to store it.
To that end, we relax \eqref{eq:general_opt_problem} by directly applying a first order approximation of $L$, that is
$
L(\mbf x , \bm \eta) \approx L(\mathbf x, \mbf 0) + \bm \eta^\T  \nabla L(\mbf x, \mbf 0) \, .
$
Using this approximation the problem from \eqref{eq:general_opt_problem} is now simplified to
\begin{equation}
\max_{\bm \eta} \nabla L({\mbf x}, \mbf 0)^\T  \bm \eta \quad \mathrm{s.t.} \quad \|\bm \eta \|_p\leq \epsilon \, ,
\label{eq:original_lin_problem}
\end{equation}
where \textcolor{myblue}{$\nabla L(\mbf x, \mbf 0) = -2  \mbf J_f(\mbf x)^T\left( \mbf y - f(\mbf x) \right)$}. Note that the attacks discussed in Section \ref{sec:avd_and_rub} for classification follow the same formulation with another choice of $L(\mbf x, \cdot)$. Therefore, the closed-form solution of \eqref{eq:original_lin_problem} can be obtained from \eqref{eq:minMainOP}. 

Unfortunately using $\mbf y \approx f(\mbf x)$ yields zero gradient in \eqref{eq:original_lin_problem}, thus leaving this approximation useless for obtaining adversarial perturbations.  This problem is tackled by taking the approximation around another random point $\tilde{\bm \eta}$ within and $\tilde \epsilon$-ball radius from $\bm \eta = \mbf 0$ as in \eqref{eq:MainOpt_lin_tilde}, with $\tilde \epsilon \leq \epsilon$. As it was mentioned above, this dithering mechanism is also used in classification problems for instance in \cite{pgd_attack}.
%
%
\section{Single Subset Attacks}\label{sec:pixel}
%
%
Another popular way of modeling undetectability, in the field of image recognition, is by constraining the number of pixels that can be modified by the attacker. This gave birth to single and multiple pixel attacks. Note that, for the case of gray-scale images, the solutions obtained in \eqref{eq:quad_sol_l1} and \eqref{eq:minMainOP}
provide already single pixels attacks. This is not true for RGB images where each pixel is represented by a subset of three values. 
Since  our analysis is not limited to image based systems, we refer to these type of attacks which target only a subset of entries as  single subset attacks.


Since perturbations belong to $\mathbb R^M$, \textcolor{myblue}{let us partition} $[M]= \{1, \dots, M \}$ into $S$ possible subsets $\mcl S_1, \dots, \mcl S_S$. The sets can  in general have different cardinalities. However, we assume here that all of them have the same cardinality of  $Z =M / S$, where $\mcl S_s = \{ i^1_s, \dots, i^Z _s \} \subseteq [M]$. \textcolor{myblue}{We define} the mixed zero-$\mcl S$ norm $\| \cdot \|_{0,\mcl S}$ of a vector, for the partition $\mcl S = \{ \mcl S_1, \dots, \mcl S_S \}$, as
the number of subsets containing at least one index associated to a non-zero entry of $\mbf x$\footnote{Similar to the so-called $\ell_0$-norm, this is not a proper norm.}:
$$
\|\mbf x \|_{0,\mcl S} = \sum_{i=1}^S \mbf 1(\|\mbf x_{\mcl S_i}\|\neq 0).
$$ 
Therefore, $\|\bm \eta \|_{0,\mcl S}$ counts the number of subsets modified by an attacker. To guarantee that only one subset is active, an additional constraint can be added to the optimization problem.  This leads to the following formulation of the single subset attack for the regression problem.
\begin{equation}\label{eq:general_pixel_attack}
\max_{\bm \eta} \| \mbf y - f(\mbf x  + \bm \eta)\|^2_2 \quad \text{s.t.} \quad \| \bm \eta \|_{\infty} \leq \epsilon \, ,  \|\bm \eta \|_{0,\mcl S} = 1 \, .
\end{equation}
A similar formulation holds as well for classification problems. 
The mixed norm $\|.\|_{0,\mcl S}$ in widely used in signal processing and compressed sensing to promoting group sparsity  \cite{rao_universal_2012}.

\subsection{Single Subset Attack for the Quadratic Problem}\label{sec:quadratic_pixel}
%
As in Section \ref{sec:quadratic}, the approximations $f(\mbf x + \bm \eta) \approx f(\mbf x) + \ma{J}_f(\mbf x)  \bm \eta$ and $\mbf y \approx f(\mbf x)$ simplify the problem \eqref{eq:general_pixel_attack} to 
\begin{equation}\label{eq:quadratic_pixel_problem}
\max_{\bm \eta} \| \ma{J}_f(\mbf x) \bm \eta \|_2^2 \quad \mathrm{s.t.} \quad \|\bm \eta \|_{\infty} \leq \epsilon \, , \, \|\bm \eta \|_{0,\mcl S} = 1 \, .
\end{equation}
As it was mentioned above, the problem is NP-hard without the mixed-norm constraint. We try to find an approximate solution to a simpler problem where only the set $\mcl S_s$ is to be modified by the attacker for $s\in[S]$. Finding the perturbation on this set amounts to solving the following problem:
\begin{equation}\label{eq:quadratic_eta_s}
\bm \eta_{s} = \argmax_{\bm \eta}\| \ma{J}_f(\mbf x) \bm \eta \|_2^2 \text{ s.t. } \|\bm \eta \|_{\infty} \leq \epsilon\, , (\bm \eta)_{i_s^z} = 0\, \, \forall i_s^z \notin \mcl S_s \, ,
\end{equation}
where $(\bm \eta)_{i_s^z}$ denotes the $i_s^z$-th entry of $\bm \eta$.
As discussed in Section \ref{sec:quadratic}, this problem is NP-hard. Since the maximization of a quadratic bowl over a box constraint lies in the corner points of the feasible set, we have:
 \[
 \bm \eta_{s} = \epsilon \sum_{z=1}^{Z} \rho_{i_s^z}^* \mbf e_{i_s^z} 
 \] 
 with $\bm \rho_s^* \triangleq (\rho_{i_s^1}^*, \dots, \rho_{i_s^Z}^*)^\T  \in \{-1, +1\}^Z$. The optimization problem can be equivalently formulated as follows:
\begin{align*}
\bm \rho_s^* &= \argmax_{\bm \rho_s \in \{-1, +1\}^Z} \left\| \ma{J}_f(\mbf x) (\epsilon \sum_{z=1}^{Z} \rho_{i_s^z} \mbf e_{i_s^z})  \right\|_2^2 \\
&= \argmax_{\bm \rho_s \in \{-1, +1\}^Z} \sum_{z=1}^{Z} \sum_{w=1}^{Z}  \rho_{i_s^z} \rho_{i_s^{w}} \ma{J}_{i_s^z}^\T  \ma{J}_{i_s^{w}} \, ,
\end{align*}
for  $\bm \rho_s \triangleq (\rho_{i_s^1}, \dots, \rho_{i_s^Z})^\T  \in \{-1, +1\}^Z$ \textcolor{myred}{and $\ma{J}_{k}$ the $k$-th column of $\ma{J}_f(\mbf x)$}.
This problem is indeed related to the well known MaxCut problem introduced by \cite{goemans1995improved}. The literature is abound with works on the MaxCut problem, the efficient solutions and their recovery guarantees. A common solution to this problem is a relaxation by a semi-definite programming problem. However, as  
we discussed  semi-definite programming solvers scales badly with the input dimension. 
Therefore, in the spirit of obtaining fast and scalable approximate solutions, that can later be used to design adversarial perturbations through iterative approximations, we propose to obtain approximate solutions using a greedy approach.  
To that end, and without loss of generality, let us assume that for a given $\mcl S_s$, the indices $i_s^1, \dots , i_s^Z\in\mcl S_s$ are sorted such that $\| \ma{J}_{i_s^1} \|_2 \geq \cdots \geq \| \ma{J}_{i_s^Z} \|_2$.  An approximate solution for $\rho_{{i}_s^z}^*$ is calculated in a greedy manner by setting $\rho_{{i}_s^1}^* =1$ and recursively calculating 
\begin{equation}
\rho_{{i}_s^z}^* = \mrm{sign}\left( \left(\sum_{j=1}^{z-1} \rho_{{i}_s^j}^* \mbf J_{{i}_s^j} \right)^\T  \mbf J_{{i}_s^z} \right) \quad \forall \, z=2,\dots, Z\, .
\label{eq:Greedy}
\end{equation}
As for greedy algorithms, this solution is fast, however, there is no optimality guarantee for it. For the case where $S=1$ and $S=M$, the expression \eqref{eq:Greedy} is an approximate solution for \eqref{eq:original_quad_opt} under the $\ell_{\infty}$-norm constraint on the perturbation (i.e., $p=\infty$).

\textcolor{myred2}{This method provides an approximate solution} to the problem for a given choice of $\mcl S_s$. The solution to \eqref{eq:quadratic_pixel_problem}  can then be obtained by solving the following problem:
\begin{align}
&\bm \eta^{*} = \bm \eta_{s*} \, , \label{eq:single_pixel_quadsol}\\
&\text{ with } 
s^* = \argmax_{s} \left\| \ma{J}_f(\mbf x) \bm \eta_s  \right\|_2^2 
\text{ and }
\bm \eta_{s} = \epsilon \sum_{z=1}^{Z} \rho_{i_s^z}^* \mbf e_{i_s^z} \nonumber
\end{align}
This is based on naive exhaustive research over the subsets which is tractable only when the number of subsets is small enough.
%


	

\subsection{Single Subset Attack for the Linear Problem}
Following the steps from Section \ref{sec:linear}, we make use of the approximation $L(\mbf x, \bm \eta) \approx L(\mbf x, \tilde{\bm \eta}) + (\bm \eta - \tilde{\bm \eta})^\T  \nabla L(\mbf x, \tilde{\bm \eta})$ which leads to the formulation of \eqref{eq:general_pixel_attack} as a linear programming problem
\begin{equation}\label{eq:linear_pixel_problem}
\max_{\bm \eta} \bm \eta^\T  \nabla L(\mbf x, \tilde{\bm \eta}) \quad \text{s.t.} \quad \|\bm \eta \|_{\infty} \leq \epsilon \, , \, \|\bm \eta \|_{0,\mcl S} = 1 \, .
\end{equation}
In the same manner as Section \ref{sec:quadratic_pixel}, for a given subset $\mcl S_s$ we define $\bm \eta_s$ as in \eqref{eq:quadratic_eta_s}. For this linear problem that results in 
\[
\bm \eta_{s} = \argmax_{\bm \eta} \nabla L(\mbf x, \tilde{\bm \eta})^\T  \bm \eta \text{ s.t. } \|\bm \eta \|_{\infty} \leq \epsilon\, , (\bm \eta)_{i_s^z} = 0\, \, \forall i_s^z \notin \mcl S_s \, .
\]
In contrast to the definition of $\bm \eta_s$ from \eqref{eq:quadratic_eta_s}, in this case we have a closed form solution for $\bm \eta_s$ as
\[
\bm \eta_{s} = \epsilon \sum_{z=1}^{Z} \mrm{sign} ((\nabla L(\mbf x, \tilde{\bm \eta}))_{i_s^z}) \mbf e_{i_s^z}\, , 
\]
which implies that $\nabla L(\mbf x, \tilde{\bm \eta})^\T  \bm \eta_s = \sum_{z=1}^{Z} \left| (\nabla L(\mbf x, \tilde{\bm \eta}))_{i_s^z} \right|$.
Therefore, the linear problem for the single subset attack \eqref{eq:linear_pixel_problem} has the closed form solution
\begin{equation}\label{eq:single_pixel_linsol}
\bm \eta^{*} = \bm \eta_{s*} \, , \text{ with } \, s^* = \argmax_{s} \sum_{z=1}^{Z} \left| (\nabla L(\mbf x, \tilde{\bm \eta}))_{i_s^z} \right|
\end{equation}
and $\bm \eta_{s} = \epsilon \sum_{z=1}^{Z} \mrm{sign} ((\nabla L(\mbf x, \tilde{\bm \eta}))_{i_s^z}) \mbf e_{i_s^z}$.
%
This results are valid for classification as well when replacing $L$ with $L(\mathbf{x}, \bm{\eta}) = -( f_{k(\mathbf{x})}(\mathbf{x} + \bm{\eta}) - \max_{l \neq k(\mathbf{x})} f_l(\mathbf{x} + \bm{\eta}) )$.
%
\section{Iterative Versions of the Linear Problem}\label{sec:iter}
%
\begin{table}[tb]
	\centering
	\begin{tabular}{c|c|c}
	Type of Attack & Relaxed Problem & Closed-Form Solution
	\\
	\hline
	$\ell_2$ $/$ $\ell_{\infty}$	
	& 
	\eqref{eq:original_lin_problem}
	& \eqref{eq:quad_sol_l2} $/$ \eqref{eq:single_pixel_quadsol}$^\diamond$
	\\ 
	constrained	
	& 
	\eqref{eq:original_quad_opt}
	& \eqref{eq:minMainOP}$^\blacktriangledown$ \\
	\hline
	Single-Subset	& 
	\eqref{eq:quadratic_pixel_problem}
	& \eqref{eq:single_pixel_quadsol}$^\diamond$\\ 
	attack	
	& 
	\eqref{eq:linear_pixel_problem}
	& \eqref{eq:single_pixel_linsol}$^\blacktriangledown$ \\
	\end{tabular}
	\vspace{0.02\linewidth}
	\caption{Summary of the obtained closed-form solutions. \textbf{Remarks}: $(\blacktriangledown)$ valid for regression and classification, $(\diamond)$ only an approximate solution.} \label{tab:summary}
\end{table}
\begin{table}
	\centering
	\begin{tabular}{c|c|c|c}
		Algorithm & Objective function $L$ & Iterative & Dithering \\
		\hline
		\ac{FGSM}  \cite{goodfellow_explaining_2014}& cross-entropy & $\times$ &  $\times$\\
		DeepFool \cite{moosavi2016deepfool} & \eqref{eq:loss_ours} with $l$ chosen using $\hat\rho_1(f)$ & $\checkmark$&  $\times$\\
		\ac{BIM} \cite{BIM2016}& cross-entropy & $\checkmark$ &  $\times$ \\
		\ac{PGD} \cite{pgd_attack} & cross-entropy & $\checkmark$& $\checkmark$ \\
		Ensemble \cite{ensemble_attack}& cross-entropy using another $f$ & $\checkmark$&  $\times$ \\
		Targeted & \eqref{eq:loss_ours} with $l$ fixed to the target & $\checkmark$ & $\checkmark$ \\
		Ours & \eqref{eq:loss_ours} & $\checkmark$ & $\checkmark$
	\end{tabular}
	\caption{Recovering Existing Attacks in Classification using this Framework.} 
	\label{tab:literature_overview}
\end{table}
In the previous sections we have formulated several variations of the problem of generating adversarial perturbations. These results are summarized in Table \ref{tab:summary}. In the same spirit as DeepFool, we make use of the obtained closed form solutions to design adversarial perturbations using iterative approximations. In Algorithm \ref{alg:iter} an iterative method based on the linear problem \eqref{eq:original_lin_problem} is introduced. This corresponds to a gradient ascent method for maximizing $L(\mbf x, \bm \eta)$ with a fixed number of iterations and steps of equal $\ell_p$-norm.
\begin{algorithm}[htb]
	\begin{algorithmic}
	\STATE \textbf{input:} $\mbf x$, $f$, $T$, $\epsilon$, $\tilde \epsilon_1, \dots, \epsilon_T$.
	\STATE \textbf{output:} $\bm \eta^*$.
	\STATE Initialize $\bm \eta_1 \leftarrow \mbf 0$.
	\FOR{$t=1, \dots, T$}
		\STATE $\tilde{\bm \eta}_t \leftarrow \bm \eta_t + \mathrm{random(\tilde \epsilon_t)}$
		\STATE $\bm \eta_t^* \leftarrow \argmax_{\bm \eta} \bm \eta^\T  \nabla L(\mbf x, \tilde{\bm \eta}_t) \text{ s.t. } \| \bm \eta \|_p \leq \epsilon/T $ (Table \ref{tab:summary})
		\STATE $\bm \eta_{t+1} \leftarrow \bm \eta_t + \bm \eta_t^* $
	\ENDFOR
	\STATE \textbf{return:} $\bm \eta^* \leftarrow \bm \eta_T$
	\end{algorithmic}
	\caption{Iterative extension for $\ell_p$ constrained methods.}
	\label{alg:iter}
\end{algorithm}
\begin{figure}[ht!]
	\centering
	\begin{minipage}[b]{0.75\linewidth}
	\centering
	\begin{tabular}{cc|cc}
	original & adv & original & adv  \\
	\hline
	\begin{minipage}[b]{0.16\linewidth}
	  \includegraphics[width=0.99\linewidth]{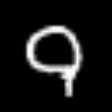}
	\end{minipage}
	&
	\begin{minipage}[b]{0.16\linewidth}
	  \includegraphics[width=0.99\linewidth]{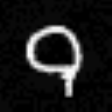}
	\end{minipage}
	&
	\begin{minipage}[b]{0.16\linewidth}
	    \includegraphics[width=0.99\linewidth]{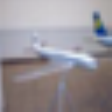}
	\end{minipage}
	&
	\begin{minipage}[b]{0.16\linewidth}
	    \includegraphics[width=0.99\linewidth]{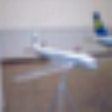}
	\end{minipage}
	\\
	nine &zero & airplane& ship \\
	\hline
	\begin{minipage}[b]{0.16\linewidth}
	  \includegraphics[width=0.99\linewidth]{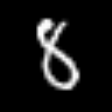}
	\end{minipage}
	&
	\begin{minipage}[b]{0.16\linewidth}
	  \includegraphics[width=0.99\linewidth]{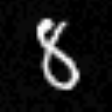}
	\end{minipage}
	&
	\begin{minipage}[b]{0.16\linewidth}
	    \includegraphics[width=0.99\linewidth]{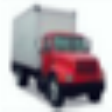}
	\end{minipage}
	&
	\begin{minipage}[b]{0.16\linewidth}
	    \includegraphics[width=0.99\linewidth]{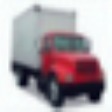}
	\end{minipage}
	\\
	eight &three & truck& car 
	\\
	\hline
	\begin{minipage}[b]{0.16\linewidth}
	  \includegraphics[width=0.99\linewidth]{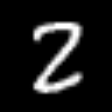}
	\end{minipage}
	&
	\begin{minipage}[b]{0.16\linewidth}
	  \includegraphics[width=0.99\linewidth]{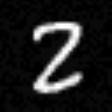}
	\end{minipage}
	&
	\begin{minipage}[b]{0.16\linewidth}
	    \includegraphics[width=0.99\linewidth]{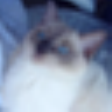}
	\end{minipage}
	&
	\begin{minipage}[b]{0.16\linewidth}
	    \includegraphics[width=0.99\linewidth]{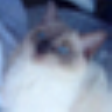}
	\end{minipage}
	\\
	two &three & cat& dog 
	\\
	\hline
	\multicolumn{2}{c}{(a) MNIST} &\multicolumn{2}{c}{(b) CIFAR-10}
	\end{tabular}
	\end{minipage}
	\caption{Examples of correctly classified images that are misclassified when adversarial noise is added using Algorithm 1.}
	\label{fig:examples2}
\end{figure}
\begin{figure}[hp]
	\centering
	\begin{minipage}[t]{.7\linewidth}
	\centerline{\includegraphics[width=.85\linewidth]{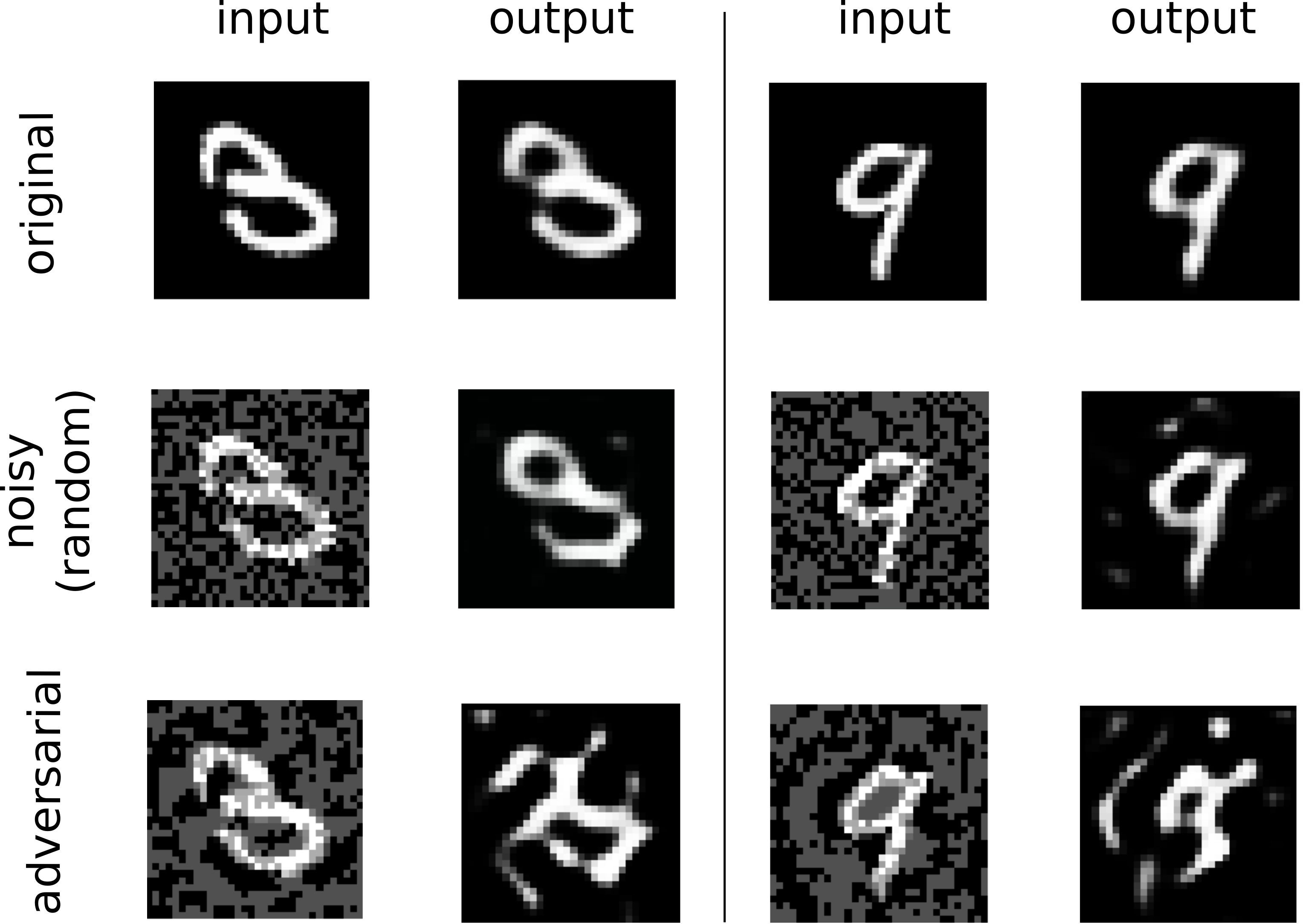}}
	\centerline{(a) Autoencoder (96\% compression)}\medskip
	\end{minipage}
	\begin{minipage}[t]{.7\linewidth}
	\centering
	\centerline{\includegraphics[width=.85\linewidth]{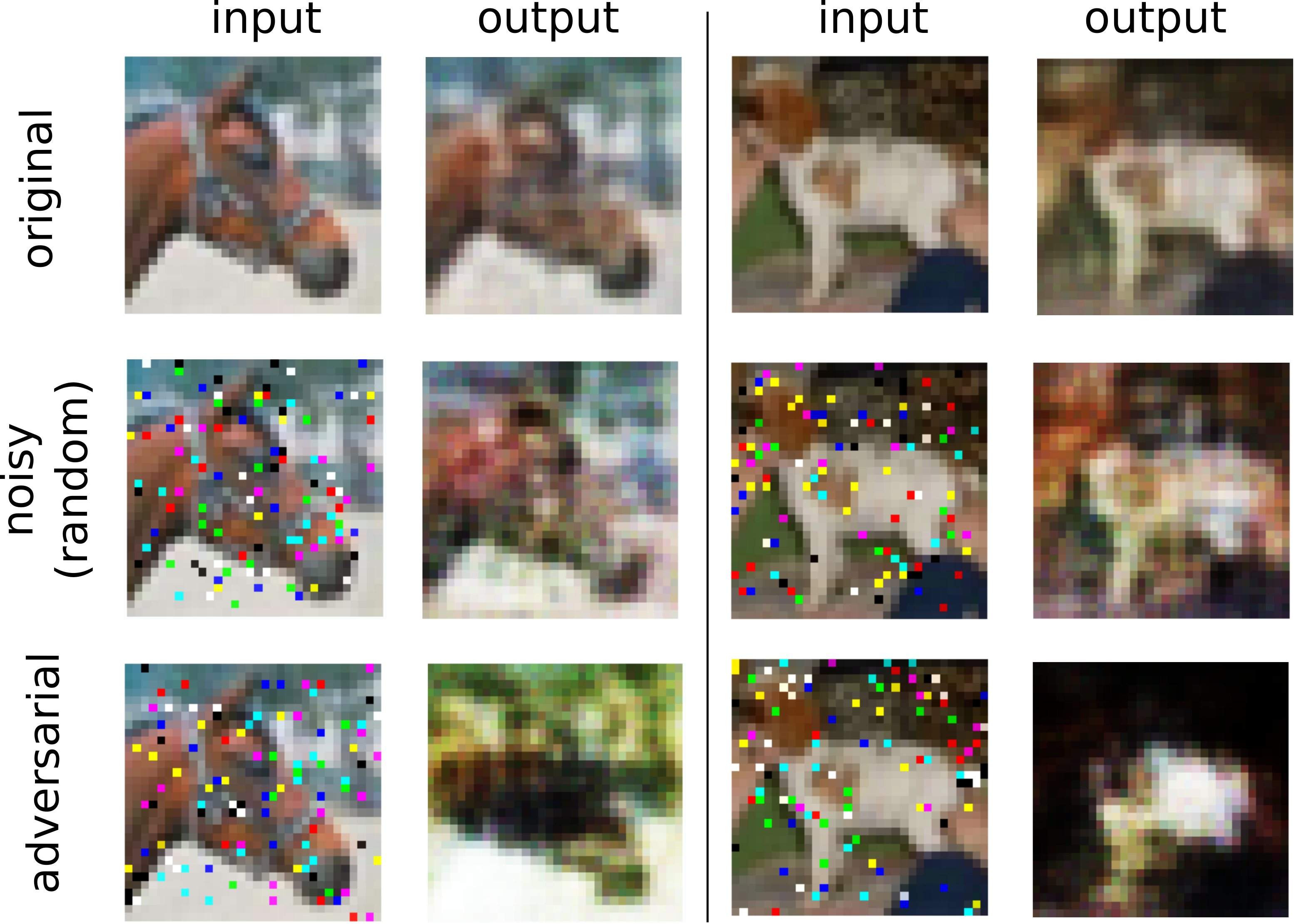}}
	\centerline{(b) Autoencoder (50\% compression)}\medskip
	\end{minipage}

	\begin{minipage}[t]{.7\linewidth}
	\centering
	\centerline{\includegraphics[width=.85\linewidth]{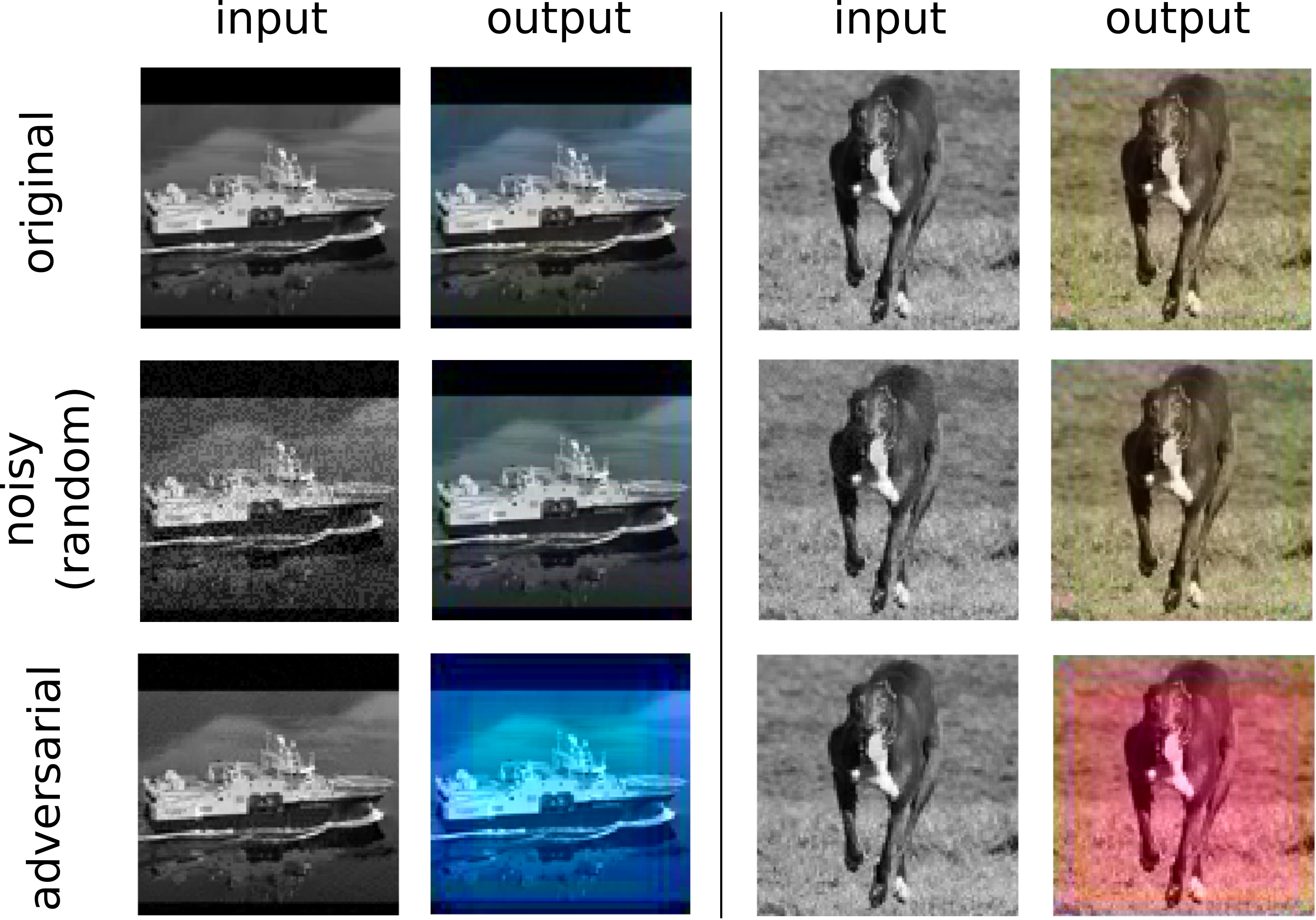}}
	\centerline{(c) Image Colorization}\medskip
	\end{minipage}
	\caption{ Adversarial examples for (a): MNIST autoencoder obtained using \textit{quadratic-$\ell_\infty$}, (b): CIFAR-10 autoencoder obtained using \textit{linear-pixel-$100$}, (c): STL-10 colorization network obtained using \textit{linear-$\ell_\infty$-$20$}.}
	\label{fig:examples}
 \end{figure}

While generalizing the results for \eqref{eq:original_lin_problem} into a gradient ascent method is trivial, the same is not true for the quadratic problem \eqref{eq:original_quad_opt}. The main reason for this is that, using the approximation $\mbf y \approx f(\mbf x)$, we were able to simplify \eqref{eq:quadratic_other_problem} into \eqref{eq:original_quad_opt} since $ \mbf y - f(\mbf x) \approx \mbf 0$. For an iterative version of this solution we must successively approximate  $f(\cdot)$ around different points $\tilde{\mbf x}$, which leads to  $ \mbf y - f(\tilde{\mbf x}) \neq \mbf 0$ even if $\mbf y = f(\mbf x)$. We leave the task of investigating alternatives for designing iterative methods with the results for \eqref{eq:original_quad_opt} for future works, and in Section \ref{sec:expres} show that the non-iterative solutions for this method are still competitive.  

Finally, replacing line 5 of Algorithm \ref{alg:iter} with
\[
\bm \eta_t^* \leftarrow \argmax_{\bm \eta} \bm \eta^\T  \nabla L(\mbf x, \tilde{\bm \eta}_t) \text{ s.t. } \| \bm \eta \|_p \leq \epsilon\, , \|\bm \eta \|_{0,\mcl S} = 1 
\]
leads to a multiple subset attack, since we modify the values of one subset at every iteration. At every iteration, we may exclude the previously modified subsets from $\mcl S$ in order to ensure that a new subset is modified.

\section{Experiments}\label{sec:expres}
\begin{figure}[ht]
	\centering
	\begin{minipage}[b]{.4\linewidth}
	\centering
	\centerline{\includegraphics[width=.99\linewidth]{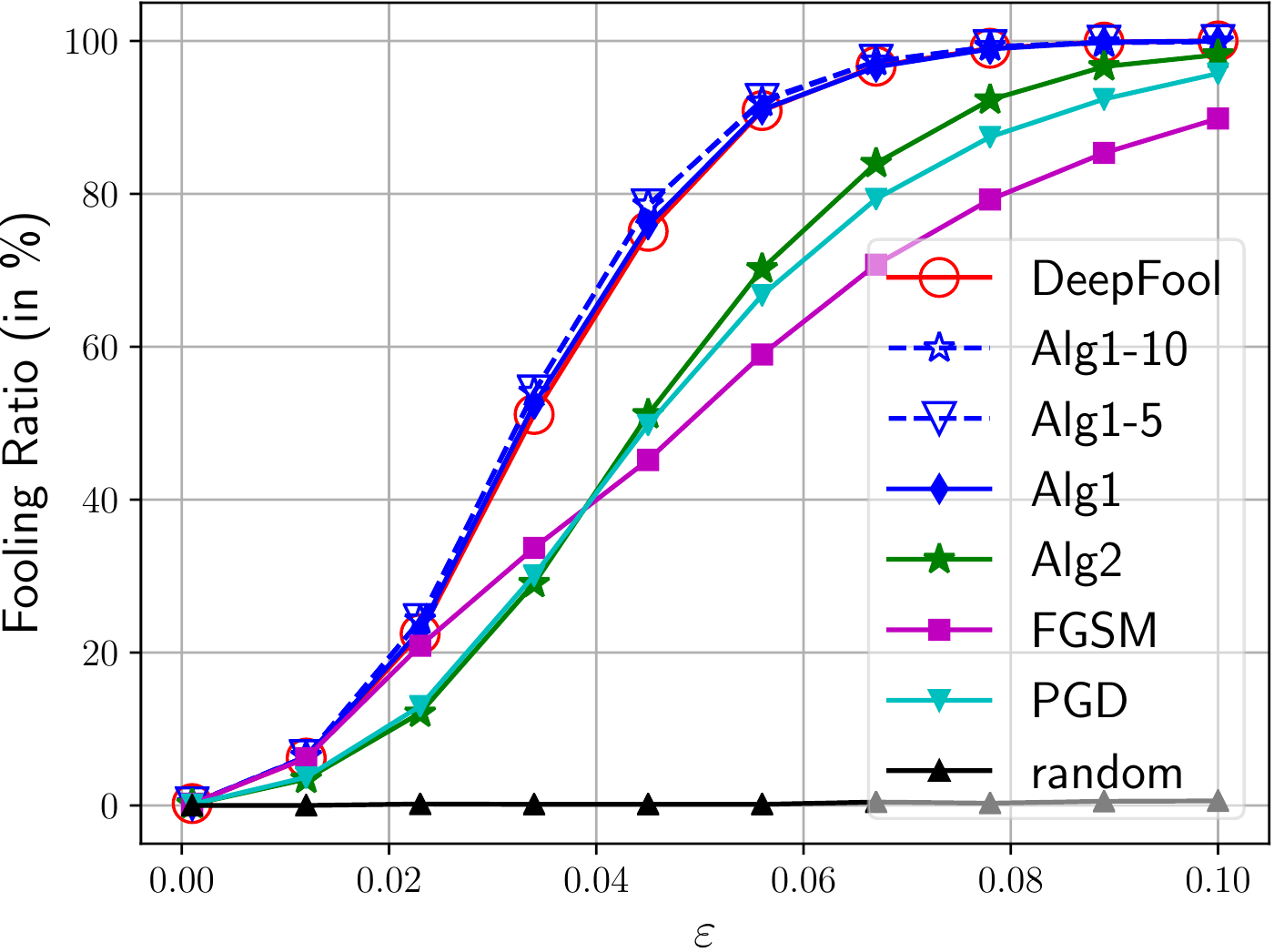}}
	\centerline{(a) FCNN}\medskip
	\end{minipage}
	\begin{minipage}[b]{.4\linewidth}
	\centering
	\centerline{\includegraphics[width=.99\linewidth]{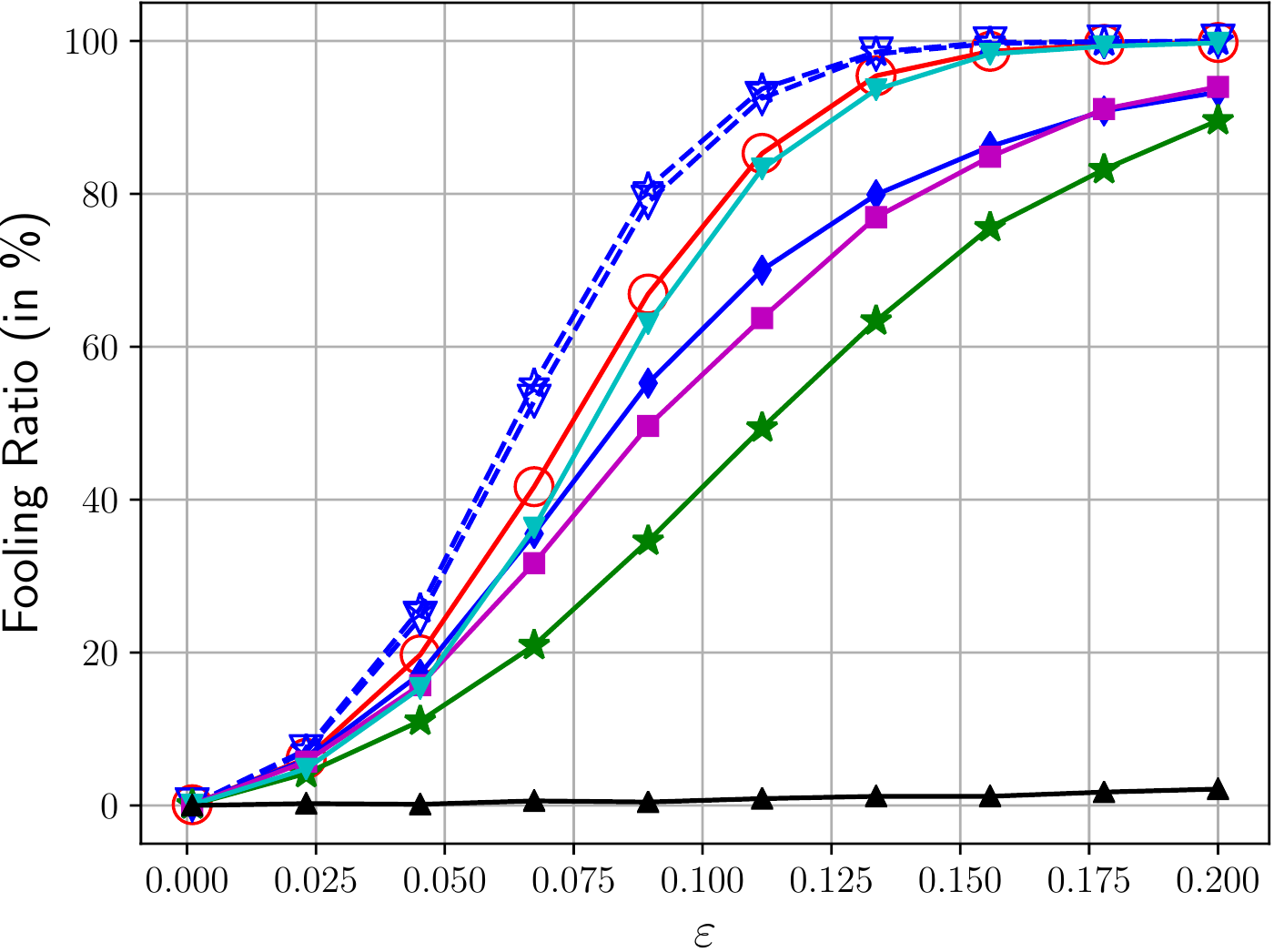}}
	\centerline{(b) LeNet-5}\medskip
	\end{minipage}

	\begin{minipage}[b]{.4\linewidth}
	\centering
	\centerline{\includegraphics[width=.99\linewidth]{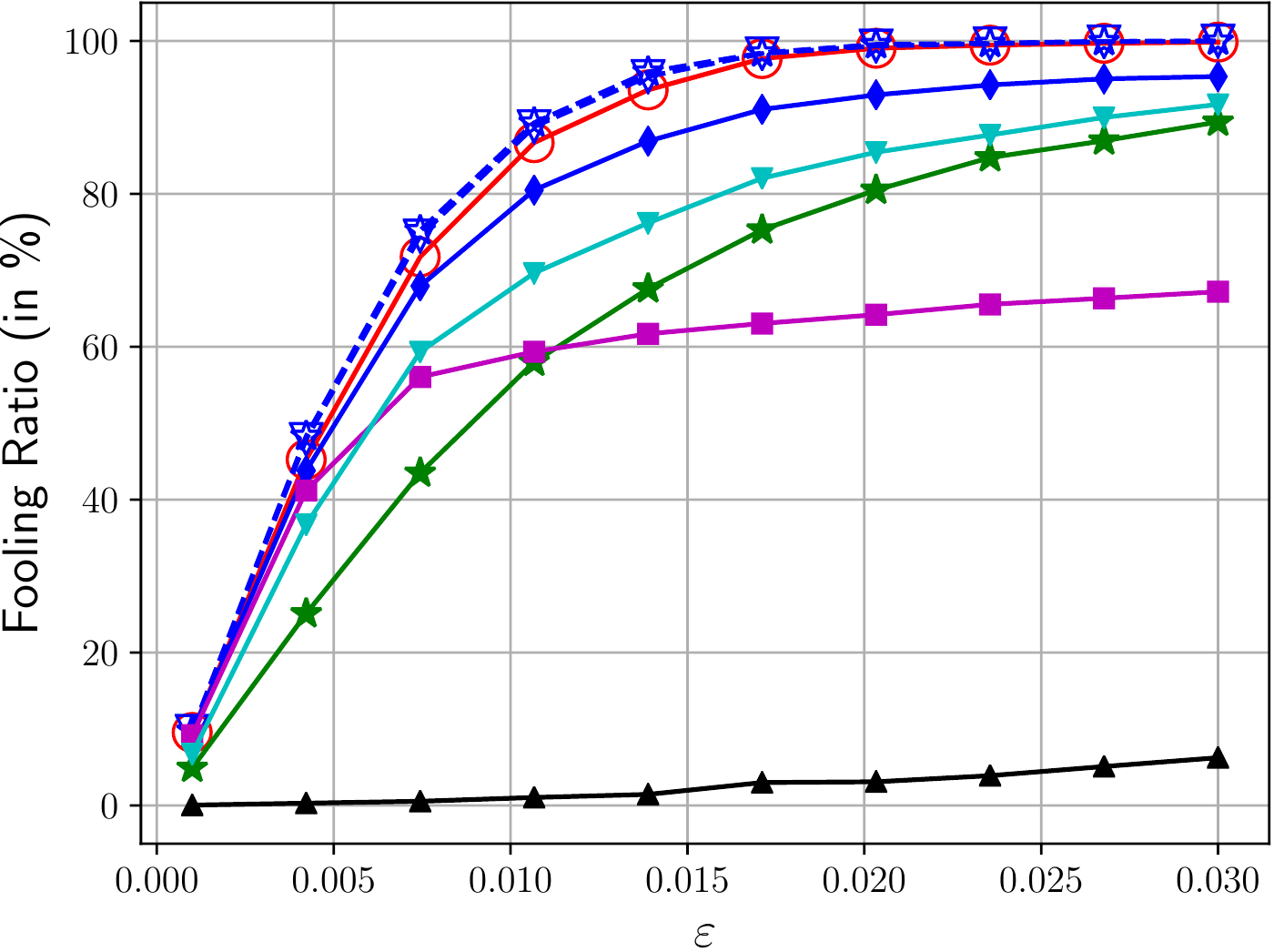}}
	\centerline{(c) NIN}\medskip
	\end{minipage}
	\begin{minipage}[b]{.4\linewidth}
	\centering
	\centerline{\includegraphics[width=.99\linewidth]{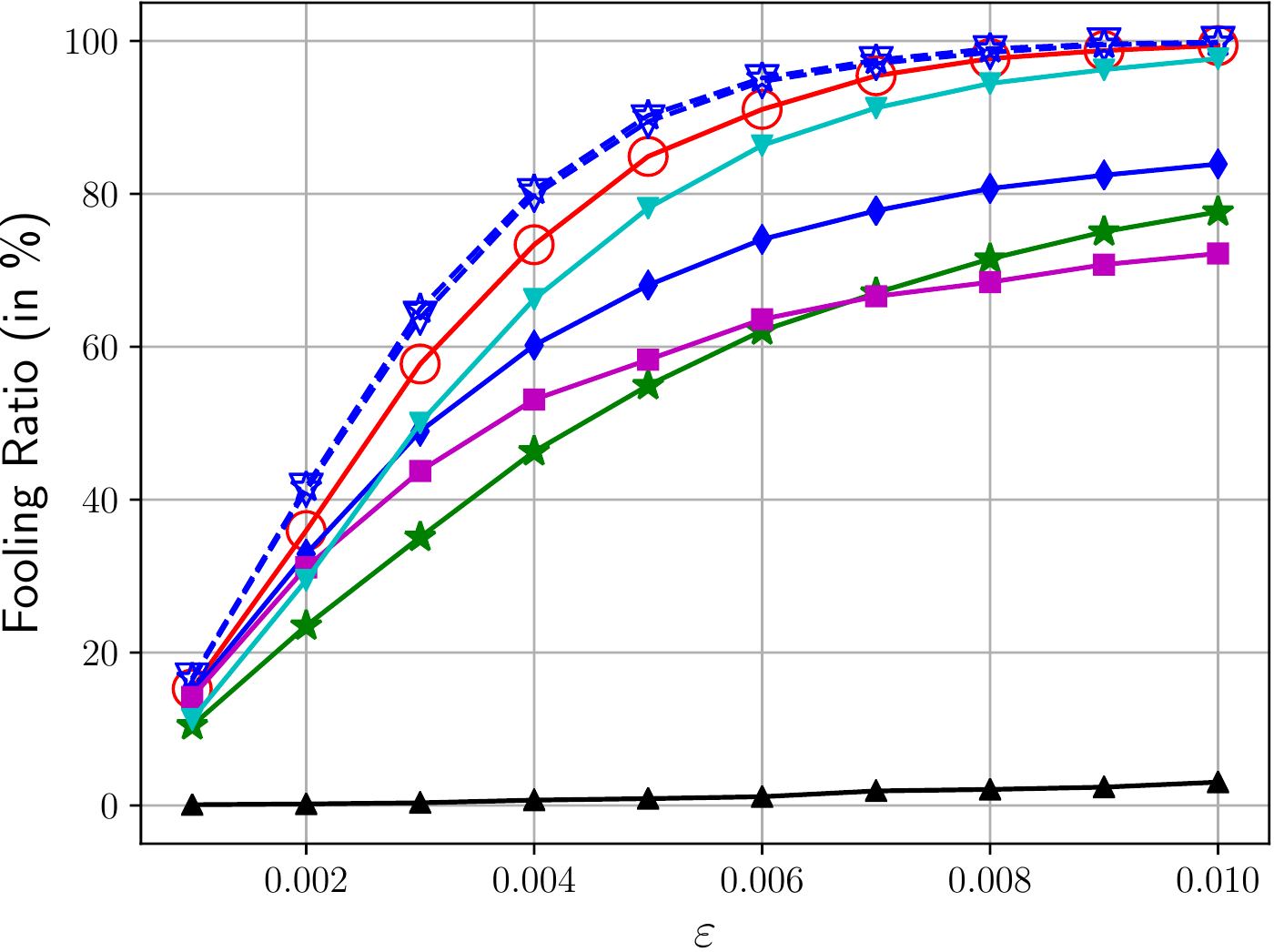}}
	\centerline{(d) DenseNet}\medskip
	\end{minipage}
	\caption{(a) and (b): Fooling ratio of the adversarial samples for different values of $\epsilon$ on the MNIST test dataset. (c) and (d): Fooling ratio of the adversarial samples for different values of $\epsilon$ on the CIFAR-10 test datasets.}
	\label{fig:res}
 \end{figure}
In this section, the proposed methods are used to fool neural networks in classification and regression problems.  The goal of this section is twofold. First, we would like to examine the performance of the newly proposed attack in classification tasks, thereby showing the utility of current adversarial generation framework. Secondly we generate adversarial perturbations for regression tasks which only received small attention in the literature. 
For this purpose we use the MNIST \cite{mnist}, CIFAR-10 \cite{cifar10}, and STL-10 datasets. 

\subsection{Classification}
As discussed in Section \ref{sec:avd_and_rub}, the appropriate loss function $L(\mathbf x, \bm \eta)$ for image classification tasks that should be used in \eqref{eq:MainOpt} is given by \eqref{eq:originalOP}. For this problem, $\| \bm{\eta} \|_{\infty} \leq \epsilon$ is a common constraint that models the undetectability, for sufficiently small $\epsilon$, of adversarial noise by an observer. However solving \eqref{eq:MainOpt} involves finding the function $L(\mathbf x, \mbf 0)$ which is defined as the minimum of $K-1$ functions with $K$ being the number of different classes. In large problems, this may significantly increase the computations required to fool one image. Therefore, we include a simplified version of this algorithm in our simulations. 
The non-iterative methods might not guarantee the fooling of the underlying network but on the other hand, the iterative methods might suffer from convergence problems.

\begin{figure}[ht!]
\centering
  \begin{minipage}[h]{.40\linewidth}
   \centering
   \centerline{\includegraphics[width=.95\linewidth]{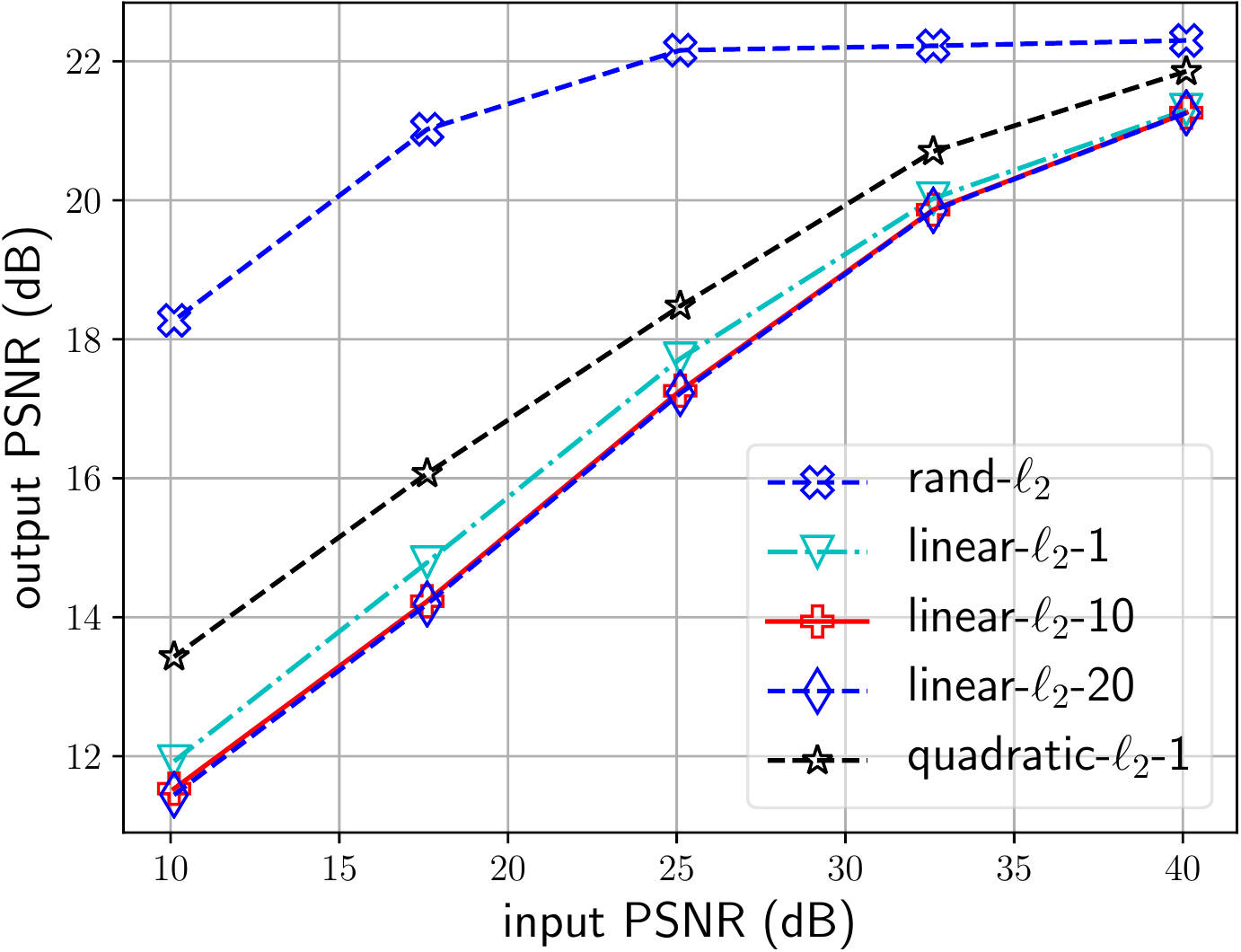}}
   \centerline{(a) MNIST: $\ell_2$ constrained}\medskip
 \end{minipage}
%
%
 \begin{minipage}[h]{.40\linewidth}
   \centering
   \centerline{\includegraphics[width=.95\linewidth]{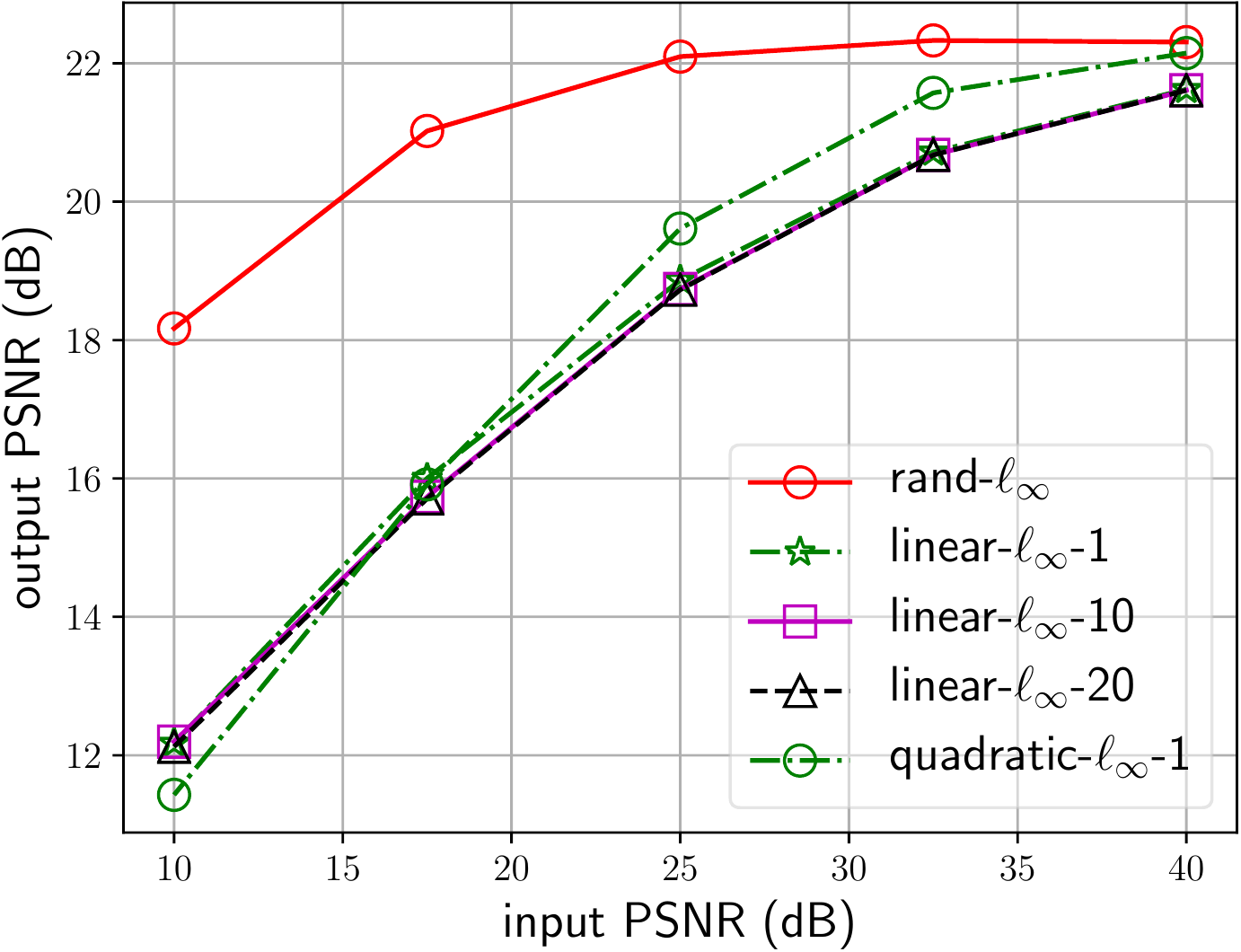}}
   \centerline{(b) MNIST: $\ell_\infty$ constrained}\medskip
 \end{minipage}

 \begin{minipage}[h]{.40\linewidth}
   \centering
   \centerline{\includegraphics[width=.95\linewidth]{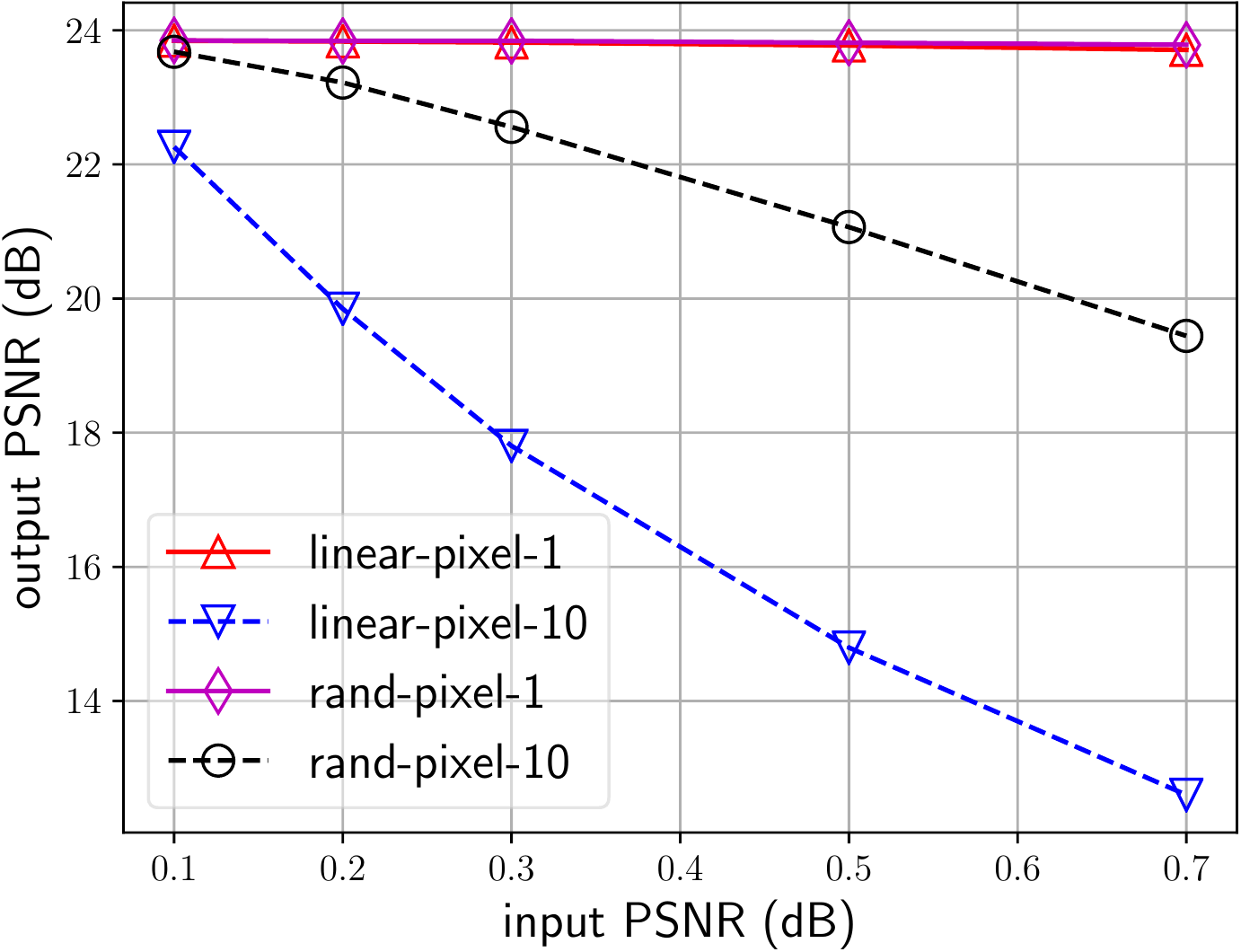}}
   \centerline{(c) CIFAR-10: Multiple pixel attack}\medskip
 \end{minipage}
 %
 %
 \begin{minipage}[h]{.40\linewidth}
   \centering
   \centerline{\includegraphics[width=.95\linewidth]{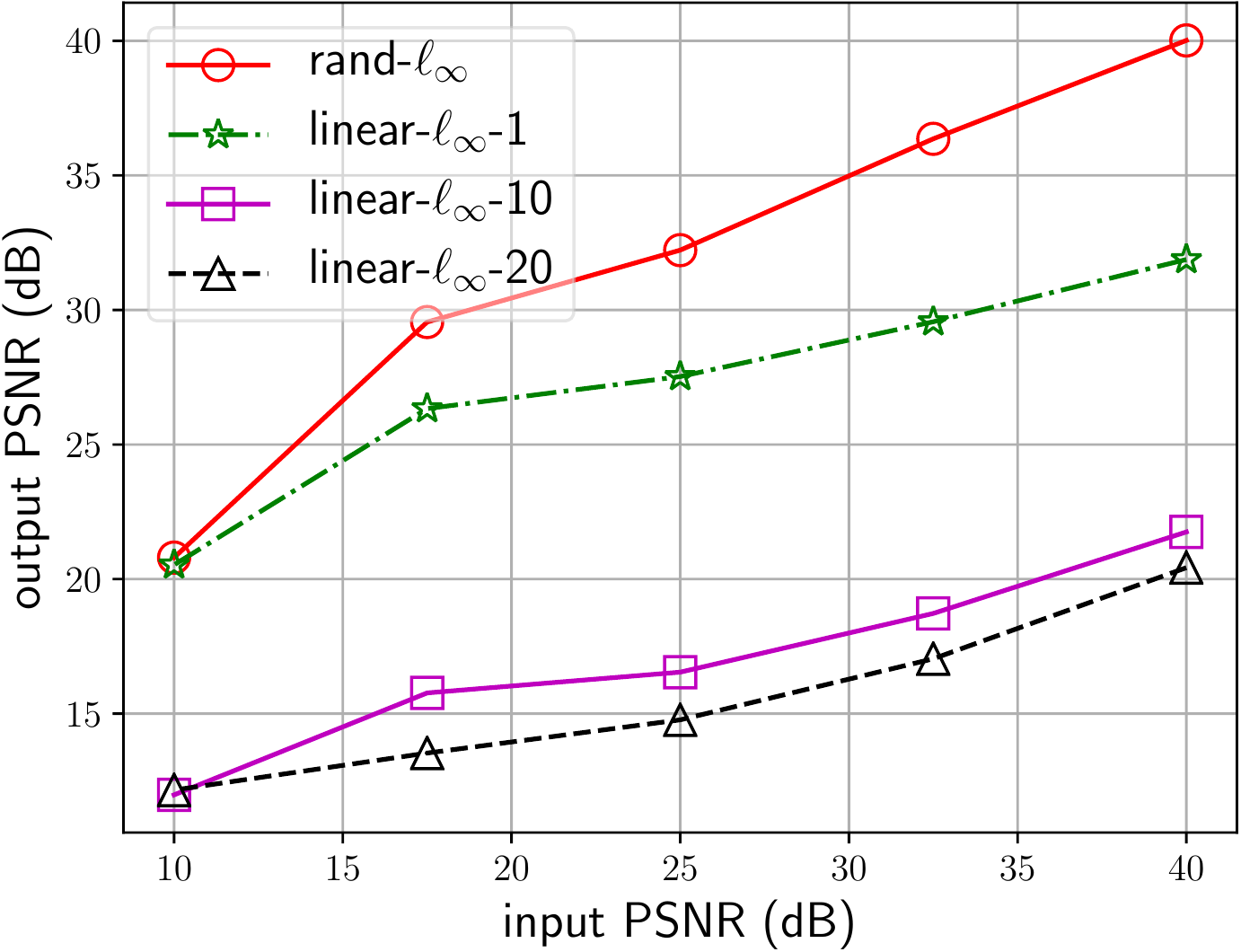}}
   \centerline{(d) STL-10 (colorization): $\ell_\infty$ constrained}\medskip
 \end{minipage}
 \caption{ Output PSNR for (a): MNIST autoencoder under $\ell_2$-norm constraint, (b): MNIST autoencoder under $\ell_\infty$-norm constraint, (c): CIFAR-10 autoencoder under multiple pixel attacks, (d): STL-10 colorization network under $\ell_\infty$-norm constraint.}
\label{fig:results}
\end{figure}
%
To benchmark the proposed adversarial algorithms, we consider following methods tested on the aforementioned datasets:
\begin{itemize}
\item \textbf{Algorithm 1}: This algorithm solves \eqref{eq:MainOpt} with $L(\mathbf x, \cdot)$ given by \eqref{eq:originalOP}. Note that, for evaluating $L$ at a given $\mathbf x$ one must search over all $l \neq k(\mathbf x)$. This can be computationally expensive when the number of possible classes (i.e., the number of possible values for $l$) is large. The $\ell_\infty$-norm is chosen for the constraint. Moreover, an example of adversarial images obtained using this algorithm is shown in Figure \ref{fig:examples}.
\item \textbf{Algorithm 1}-$T$: This is the iterative version of Algorithm 1 with $T$ iterations. The adversarial perturbation is the sum of $T$ perturbation vectors with $\ell_\infty$-norm of $\epsilon / T$ computed through $T$ successive approximations.
\item \textbf{Algorithm 2}: This algorithm approximates \eqref{eq:originalOP} with $L(\mathbf x, \bm \eta) \approx f_{k(\mathbf{x})}(\mathbf{x} + \bm \eta)$, thus reducing the computation of $L(\mathbf x)$ when the number of classes is large. Note that we cannot use $L(\mathbf x, \bm{\eta})<0$ to guarantee that we have fooled the network. Nevertheless, the lower the value of $L(\mathbf x , \bm{\eta})$ the most likely it is that the network has been fooled. The same reasoning is valid for the FGSM algorithm.
\item \textbf{FGSM}: This well-known method was proposed by \cite{goodfellow_explaining_2014} where $L(\mbf x, \bm \eta)$ is replaced by the negative training loss for the input $\mbf x + \bm \eta$. Usually the cross-entropy loss is used for this purpose. With the newly replaced function,  \eqref{eq:MainOpt} is solved for $p=\infty$.
\item \textbf{PGD}: This method is the iterative version of \ac{FGSM} ($T > 1$) with $\tilde \epsilon_1 = \epsilon$ and $\tilde \epsilon_t = 0 $ for all $t>1$. It constitutes one of the state of the art attacks in the literature.
\item \textbf{DeepFool}: This method was proposed in \cite{moosavi2016deepfool} and makes use of iterative approximations. Every iteration of DeepFool can be written within our framework by replacing $L$ by 
\begin{align*}
L(\mathbf x , \bm{\eta} ) &= f_{k(\mathbf{x})}(\mathbf{x} + \bm{\eta}) - f_{\hat l}(\mathbf{x} + \bm{\eta}) \, , \quad
\text{where} \\
\hat l &= \argmin_{l \neq k(\mathbf x)} \left\{ \frac{|f_{k(\mathbf{x})}(\mathbf{x}) - f_l(\mathbf{x})|}{ \|\nabla f_{k(\mathbf{x})}(\mathbf{x}) - \nabla f_l(\mathbf{x})\|_q } \right\} \, .
\end{align*}
The adversarial perturbations are computed using $p = \infty$, thus $q=1$, with a maximum of $50$ iterations. These parameters were taken from \cite{moosavi2016deepfool}. Note that $\hat l$ is chosen to minimize the robustness $\hat\rho_1(f)$ for $\mcl D = \{\mbf x\}$.
\item \textbf{Random}: For benchmarking purposes, we also consider random perturbations with independent Bernoulli distributed entries with \textcolor{myblue}{$\mathbb P (\epsilon) = \mathbb P (-\epsilon) = \frac{1}{2}$}. This helps to demarcate the essential difference of adversarial and random perturbations.
\end{itemize}
Note that these methods from the literature can be expressed in terms of the proposed framework as summarized in Table \ref{tab:summary}. In that table we also include the black-box ensemble attack from \cite{ensemble_attack} and targeted attacks \cite{carlini2017towards,papernot2016limitations,2017arXiv170309387B,cisse2017houdini,sarkar2017upset}. In \cite{ensemble_attack} the target neural network function is not known thus another known neural network function $f$ is used instead, hoping that the obtained adversarial example transfers to the unknown network. Targeted attacks are used when the objective is to generate adversarial examples that are classified by target system as belonging to some given target class $\ell \in [K]$. That corresponds to fixing the $l$ in \eqref{eq:loss_ours}, that is $L(\mbf x, \bm \eta) = f_{k(\mbf x)}(\mbf x + \bm \eta) - f_{l}(\mbf x + \bm \eta)$.

The above methods are tested  on the following deep neural network architectures:
\begin{itemize}
  \item \textbf{MNIST} : A fully connected network with two hidden layers of size $150$ and $100$ respectively, as well as the LeNet-$5$ architecture \cite{lenet}. 
  \item \textbf{CIFAR-10} : The Network In Network (NIN) architecture \cite{nin}, and a $40$ layer DenseNet \cite{densenet}.
\end{itemize}

As a performance measure, we use the \emph{fooling ratio} defined in \cite{moosavi2016deepfool} as the percentage of correctly classified images that are missclassified when adversarial perturbations are applied. Of course, the fooling ratio depends on the constraint on the norm of adversarial examples. 
Therefore, in Figure \ref{fig:res} we observe the fooling ratio for different values of $\epsilon$ on the aforementioned neural networks. 
As expected, the increased computational complexity of iterative methods such as DeepFool and Algorithm 1-$T$ translates into increased performance with respect to non-iterative methods. 
Nevertheless, as shown in Figures \ref{fig:res}(a) and (c), the performance gap between iterative and non-iterative algorithms is not always significant. 
For the case of iterative algorithms, the proposed Algorithm 1-$T$ outperforms DeepFool and \ac{PGD}. The same holds true for Algorithm 1 with respect to other non-iterative methods such as FGSM, while Algorithm 2 obtains competitive performance with respect to FGSM. However, note that adversarial training using \ac{PGD} is the state of the art defense against adversarial examples, thus \ac{PGD} may still be a better choice than Algorithm 1-$T$ for adversarial training. 

Finally, we measure the robustness of different networks using $\hat{\rho}_1(f)$ and $\hat{\rho}_2(f)$, with $p=\infty$. We also include the minimum $\epsilon$, such that DeepFool obtains a fooling ratio greater than 99\%, as a performance measure as well. These results are summarized in Table \ref{tab:robust}, where we obtain coherent results between the $3$ measures.
\begin{table}[htb]
    \centering
    \begin{tabular}{l|c|c|c|c}
    & \small{Test} & \small{$\hat{\rho}_1(f)$} & \small{$\hat{\rho}_2(f)$}  & \small{fooled} \\
    & \small{error} &\small{\cite{moosavi2016deepfool}} & \small{(ours)}  & \small{$>$99\%} \\
    \hline
    \small{FCNN (MNIST)}& \small{1.7\%} & 0.036 & 0.034  & $\epsilon =$0.076 \\
    \small{LeNet-5 (MNIST)}& \small{0.9\%} & \textbf{0.077} & \textbf{0.061}  & $\epsilon =$\textbf{0.164} \\
    \hline
    \small{NIN (CIFAR-10)}& \small{13.8\%} & \textbf{0.012} & \textbf{0.004}  & $\epsilon =$\textbf{0.018} \\
    \small{DenseNet (CIFAR-10)}& \small{5.2\%} & 0.006 & 0.002  & $\epsilon =$0.010
    \end{tabular}
    \caption{Robustness measures for different classifiers.}
    \label{tab:robust}
\end{table}
%
%
\subsection{Regression} 
For the sake of clarity we use the notation \textit{quadratic-$\ell_p$} to denote the method of computing adversarial perturbations by solving the quadratic problem \eqref{eq:original_quad_opt} under the $\ell_p$-norm constraint. 
In the same manner, Algorithm \ref{alg:iter} with $T$ iterations and the $\ell_p$-norm constraint is referred to as \textit{linear-$\ell_p$-$T$}. Since the experiments carried out in this section are exclusively image based, we use the notation \textit{linear-pixel-$T$} to denote the multiple subset attack with  $\|\bm \eta \|_{0,\mcl S} = T$.
Since the aim of the proposed attacks is to maximize the MSE of the target system, we use the Peak-Signal-to-Noise Ratio (PSNR), which is a common measure for image quality and is defined as $\mrm{PSNR} = \text{(maximum pixel value)}^2 / \mrm{MSE}$, as the performance metric.

Similarly to \cite{moosavi2016deepfool}, we show the validity of our methods by comparing their performance against appropriate types of random noise. For $p=2$ the random perturbation is computed as $\bm \eta = \epsilon \, \mbf w / \| \mbf w\|_2$, where the entries of $\mbf w$ are independently drawn from a Gaussian distribution. For $p=\infty$ the random perturbation $\bm \eta$ has independent Bernoulli distributed entries with $\mbb P(\epsilon) = \mbb P(-\epsilon) = 1/2$. In the case of multiple subset attacks we perform the same approach as for $p = \infty$ but only on $T$ randomly chosen pixels, while setting the other pixels of $\bm \eta$ to zero.
In order to keep a consistent notation, we refer to these $3$ methods of generating random perturbations as \textit{random-$\ell_2$}, \textit{random-$\ell_\infty$}, and \textit{random-pixel-$T$} respectively.
For our experiments we use the MNIST, CIFAR-10 and STL-10 datasets. A different neural network is trained for each of these datasets. As in \cite{tabacof2016adversarial}, we also consider  autoencoders. For MNIST and CIFAR-10 we have trained fully connected autoencoders with $96 \%$ and $50\%$ compression rates respectively. In addition, we go beyond autoencoders and train the image colorization architecture from \cite{koala} for the STL-10 dataset. Different example images obtained from applying the proposed methods on these networks are shown in Figure \ref{fig:examples}. For instance, in Figure \ref{fig:examples}(a) we observe that the autoencoder trained on MNIST is able to denoise random perturbation correctly but fails to do so with adversarial perturbations obtained using the \textit{quadratic-$\ell_\infty$} method. Similarly, in Figure \ref{fig:examples}(b), the \textit{random-pixel-$100$} algorithm distorts the output significantly more than its random counterpart. These two experiments align with the observation of \cite{tabacof2016adversarial} that autoencoders tend to be more robust to adversarial attacks than deep neural networks used for classification. The deep neural network trained for colorization is highly sensitive to adversarial perturbations as illustrated in Figure \ref{fig:examples}(c), where the original and adversarial images are nearly identical.

While the results shown in Figure \ref{fig:examples} are for some particular images, in Figure \ref{fig:results} we measure the performance of different adversarial attacks using the average output PSNR over $20$ randomly selected images from the corresponding datasets. In Figures \ref{fig:results}(a) and \ref{fig:results}(b) we observe how computing adversarial perturbations through successive linearizations improves the performance. This behavior is more pronounced in Figure \ref{fig:results}(d), where iterative linearizations are responsible for more than $10$ dB of output PSNR reduction. Note that, in Figures \ref{fig:results}(a) and \ref{fig:results}(b) the non-iterative \textit{quadratic-$\ell_p$} algorithm performs competitively, even when compared to iterative methods. In Figure \ref{fig:results}(b) we observe that the autoencoder trained on CIFAR-10 is robust to single pixel attacks. However, an important degradation of the systems performance, with respect to random noise, can be obtained through adversarial perturbations in the $100$ pixels attack ($\approx 9.7\%$ of the total number of pixels). Finally, in Figure \ref{fig:results}(d), we can clearly observe the instability of the image colorization network to adversarial attacks. 
These experiments show that, even though autoencoders are somehow robust to adversarial noise, this may not be true for deep neural networks in other regression problems.

\section{Conclusion}
The perturbation analysis of different learning algorithms leads to a framework for generating adversarial examples via convex programming. For classification we have formulated already existing methods as special cases of the proposed framework as well as proposing novel methods for designing adversarial perturbations under various desirable constraints. This includes in particular single-pixel and single-subset attacks. The framework is additionally used to demonstrate adversarial vulnerability of regression algorithms by generating adversarial perturbations.  
We numerically evaluate the applicability of this framework first by benchmarking the newly introduced algorithms for classification through empirical simulations of the fooling ratio benchmarked against the well-known \ac{FGSM}, DeepFool, and \ac{PGD} methods.
Through experiments we have shown the existence of adversarial examples in regression for the case of autoencoders and image colorization tasks. 

\newpage 
\bibliographystyle{alpha}
\bibliography{sample}

\end{document}